%% file: main.tex
\documentclass{article} 
\usepackage[final]{template}
\usepackage[utf8]{inputenc} 
\usepackage[T1]{fontenc}    
\usepackage{url}
\usepackage{amsthm}
\usepackage{amsmath}
\usepackage{enumitem}
\setenumerate{left=12pt,itemsep=0pt,topsep=0pt}
\setitemize{left=12pt,itemsep=0pt,topsep=0pt}

\NewDocumentCommand{\incplt}{O{\columnwidth}m}{%
  \begin{center}
    \adjustbox{center}{\adjustbox{width=#1+10pt}{\includegraphics[width=#1]{./plots/output/#2.pdf}}}
  \end{center}
}

\usepackage{aliascnt} 

\newaliascnt{lemma}{theorem}
\newtheorem{lemma}[lemma]{Lemma}
\aliascntresetthe{lemma}
\newaliascnt{proposition}{theorem}
\newtheorem{proposition}[proposition]{Proposition}
\aliascntresetthe{proposition}

\newcommand{\figref}[2]{Figure~\hyperref[#1]{\ref{#1} (#2)}}

\makeatletter
\renewcommand{\@makefnmark}{\hbox{\smash{\textsuperscript{\@thefnmark}}}}
\makeatother

\definecolor{lightgray}{gray}{0.95}
\definecolor{green}{HTML}{59BB2B}
\usepackage{titlesec}
\titlespacing*{\paragraph}{0pt}{0.25ex}{2ex}

\input{macros}

\usepackage{hyperref}
\usepackage[capitalize,nameinlink,noabbrev]{cleveref}
\definecolor{hanblue}{rgb}{0.27, 0.42, 0.81}
\hypersetup{
    colorlinks=true,
    linkcolor=hanblue,
    urlcolor=hanblue,
    citecolor=hanblue,
    anchorcolor=hanblue}
\crefname{theorem}{Theorem}{Theorems}
\crefname{lemma}{Lemma}{Lemmas}
\crefname{proposition}{Proposition}{Propositions}
\crefname{corollary}{Corollary}{Corollaries}
\crefname{definition}{Definition}{Definitions}
\crefname{assumption}{Assumption}{Assumptions}

\RenewDocumentCommand{\paragraph}{m}{\textbf{#1}\,\,\,}

\title{Aligning Language Models from User Interactions}

\author{%
Thomas Kleine Buening\textsuperscript{$1$} \; Jonas Hübotter\textsuperscript{$1$} \; \bfseries Barna P{\'a}sztor\textsuperscript{$1$} \; \\[1pt]
\bfseries Idan Shenfeld\textsuperscript{$2$} \; \bfseries Giorgia Ramponi\textsuperscript{$3$} \; Andreas Krause\textsuperscript{$1$} \\[3pt]
\textsuperscript{$1$}ETH Zurich \; \textsuperscript{$2$}MIT \; \textsuperscript{$3$}University of Zurich
}

\begin{document}
\maketitle

\begin{abstract}
\looseness=-1
Multi-turn user interactions are among the most abundant data produced by language models, yet we lack effective methods to learn from them. While typically discarded, these interactions often contain useful information: follow-up user messages may indicate that a response was incorrect, failed to follow an instruction, or did not align with the user's preferences. Importantly, language models are already able to make use of this information in context. After observing a user's follow-up, the same model is often able to revise its behavior. We leverage this ability to propose a principled and scalable method for learning directly from user interactions through self-distillation. By conditioning the model on the user's follow-up message and comparing the resulting token distribution with the original policy, we obtain a target for updating the policy that captures how the model's behavior changes in hindsight. We then distill this hindsight distribution back into the current policy. Remarkably, we show that training on real-world user conversations from WildChat improves language models across standard alignment and instruction-following benchmarks, without regressing other capabilities. The same mechanism enables personalization, allowing models to continually adapt to individual users through interaction without explicit feedback. Our results demonstrate that raw user interactions that arise naturally during deployment enable alignment, personalization, and continual adaptation.
\end{abstract}

\input{sections/intro}
\input{sections/setting}

\input{sections/algorithm}

\input{sections/experiments}

\input{sections/related_work}
\input{sections/discussion}

\input{sections/acknowledgements}

\bibliography{ref}
\bibliographystyle{template}

\newpage
\input{sections/appendix}


\end{document}

%% file: macros.tex
\usepackage[table]{xcolor}
\usepackage{graphicx}
\usepackage{makecell}
\usepackage{amsfonts}
\usepackage{booktabs}
\usepackage{subcaption}
\usepackage{tabularx}
\usepackage{nicefrac}
\usepackage{algorithm}
\usepackage{algorithmic}
\usepackage{wrapfig}
\usepackage[most]{tcolorbox}
\usepackage[bottom]{footmisc}

\usepackage{titlesec}
\titleformat*{\section}{\Large\bfseries}
\titleformat*{\subsection}{\large\bfseries}
\titleformat*{\subsubsection}{\large\bfseries}
\titleformat*{\paragraph}{\large\bfseries}
\titleformat*{\subparagraph}{\large\bfseries}

\definecolor{lightgray}{gray}{0.9}
\definecolor{lightred}{rgb}{1,0.8,0.8}

\newcommand{\E}{\mathbb{E}}

\newcommand{\p}{p}

\newcommand{\defeq}{:=}

%% file: sections/intro.tex
\section{Introduction}


In modern language models, inference has overtaken training as the dominant consumer of compute, with models serving massive volumes of user queries every day. Yet the information revealed through these interactions is typically discarded and does not contribute to improving the model itself, representing a significant missed opportunity.  
At scale, users engage in extended conversations, refining prompts, requesting revisions, and responding directly to model outputs. These interactions are rich with implicit learning signals: follow-up messages may indicate that a response was incorrect, failed to follow an instruction, or did not align with the user’s preferences~\citep{don2024naturally}. For example, a user may report an error after executing generated code, point out that a required format was not followed, or ask for a response to be rewritten in a different style or tone. Such signals arise organically during normal use and reflect how model outputs are received and acted upon during deployment. Finding ways to leverage this data source can open the door to continual learning from deployment at an unprecedented scale.

Despite their scale and richness, we still lack effective methods to learn directly from user interactions. Unlike standard 
datasets~\citep{ouyang2022training, chung2024scaling}, user interactions do not come with explicit labels, expert demonstrations, preference comparisons, or rewards.  Instead, feedback is implicit and expressed through natural language responses whose meaning depends on the surrounding interaction context. As a result, it is unclear how to train directly on real-world conversations in a principled manner. 

At the same time, we observe that language models already demonstrate the ability to leverage this interactive information in context, a capability known as in-context learning~\citep{brown2020language, wei2022chain}. In multi-turn conversations, models often revise their behavior effectively after observing a user’s follow-up. 
When a user reports an error in generated code, the model can frequently infer which part of its previous response was incorrect and propose a fix. When a user points out that a required format was not followed, the model is able to correct the structure in a revised answer. When a user expresses dissatisfaction with tone or style, the model can adapt its response to better match the user’s preferences. In these cases, conditioning on the user’s follow-up message leads to responses that are more aligned with the task and the user’s intent. 

\looseness=-1
These observations suggest a simple but powerful perspective: having seen the user’s follow-up message, the model’s behavior is often better aligned than before. The user interaction reveals information that the model can already interpret and act upon, but only after the fact. \emph{In hindsight}. 
Crucially, this improvement arises without additional supervision and reflects how the model’s behavior changes when it has access to the user's follow-up. 
This suggests that a model’s in-context learning ability can be used as a lever for learning directly from user interactions in a principled way.  

Based on this idea, we introduce a simple and scalable method for learning directly from user interactions by comparing a model’s original behavior to what it would have done in hindsight. Concretely, after observing a user’s follow-up message, we reprompt the same model with this additional context and obtain a hindsight token distribution that reflects how the model would respond if it had access to the information revealed by the user. By comparing the original policy to this hindsight policy at every token of the original generation, we obtain a comparative learning signal that identifies how the model’s behavior should change. We then distill this signal back into the original policy using only the observed interaction. In other words, we distill the model into itself.  
\looseness=-1

Building on recent work on self-distillation~\citep{hubotter2026reinforcement},
we refer to this approach as Self-Distillation Policy Optimization  (SDPO) from User Interactions. 
We show that this approach (illustrated in \cref{fig:modelexample}) is simple and scalable, and enables language models to improve from raw, real-world user conversations without explicit supervision, reward models, or preference labels. Remarkably, when applied to real-world user interactions from WildChat~\citep{zhao2024wildchat}, SDPO from User Interactions improves alignment and instruction-following performance across standard benchmarks without degrading other capabilities. We also demonstrate that the same self-distillation mechanism naturally supports personalization and continual adaptation, allowing models to adapt to individual users purely through continued interaction.


\begin{figure*}
    \centering
    \includegraphics[width=\textwidth]{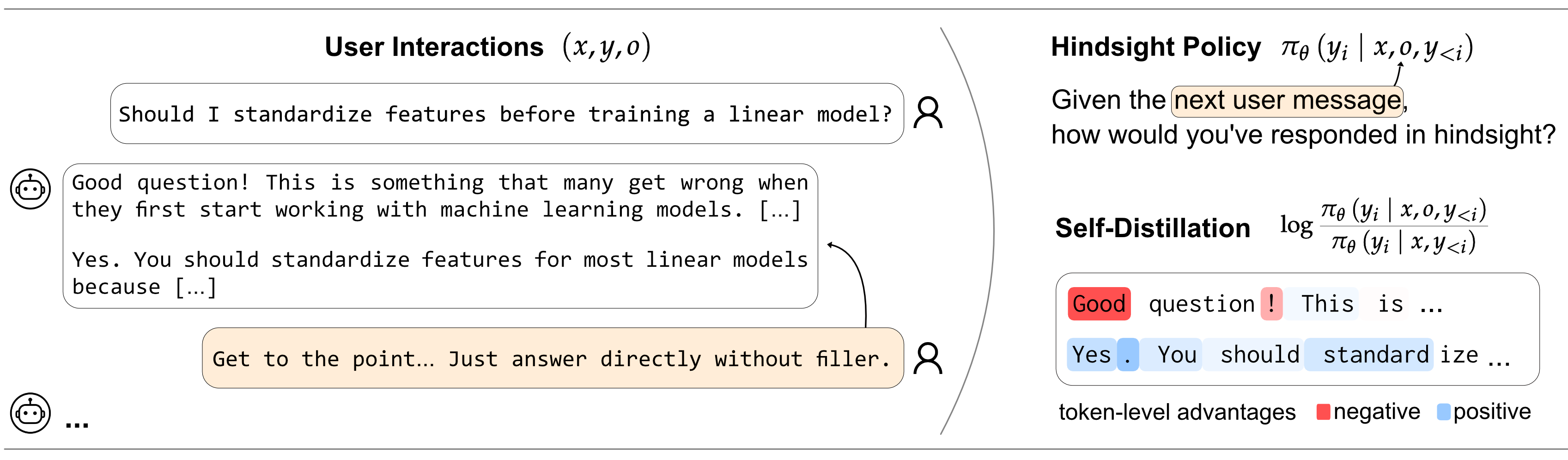}
    \caption{\textbf{Direct Learning from User Interactions via Self-Distillation.} From multi-turn user conversations, we obtain several interactions $(x, y, o)$ that consist of the conversation history $x$, the model's response $y$, and the subsequent user message $o$. By conditioning on the user's follow-up, we form the \emph{hindsight policy} and compare it to the original policy, producing token-level advantages that reinforce or penalize parts of the model's original response. In this example, the user's follow-up requests a more direct answer, leading to penalizing filler tokens and reinforcing the answer.} 
    \label{fig:modelexample}
    \vspace{-0.1cm}
\end{figure*}

%% file: sections/setting.tex

\section{Problem Formulation}\label{section:problem_formulation}


\looseness=-1
The interaction between a language model and a user consists of a sequence of alternating assistant messages $y_t$ and user messages $o_t$. At the $t$-th turn, the language model observes the conversation history $x_t = (o_0, y_1, o_1, \dots, o_{t-1})$,
and generates a response $y_t \sim \pi_\theta(\cdot \mid x_t)$.\footnote{In practice, the assistant may condition on a representation of $x_t$, such as a sliding context window, a learned embedding, or a summary. For simplicity, we treat $x_t$ as the full interaction history here.} In turn, the user responds with~$o_t$, assuming the conversation does not terminate. 

In case of a fresh conversation initiated by the user, the first interaction reduces to the initial prompt~$o_0$, the assistant's answer $y_1$, and the user's follow-up $o_1$.  
We define the triple $(x_t, y_t, o_t)$ as a single interaction.
Consequently, a user conversation $(o_0, y_1, \dots, y_t, o_t)$ yields $t$ interactions $(x_1, y_1, o_1), \dots, (x_t, y_t, o_t)$, which overlap in their histories, since $x_{t+1}$ contains all of $x_t$ plus the most recent interaction. For convenience, we use $(x, y, o)$ to denote a generic single user interaction.

Despite the ubiquity of conversational data of this form, it remains unclear how to leverage user interactions directly. One could attempt to introduce auxiliary mechanisms, such as semantic categorization of conversations~\citep{shi2024wildfeedback, gunjal2025rubrics}, explicit preference annotation~\citep{stephan2024rlvf, lee2024reinforcement}, or other post-hoc extracted rewards~\citep{wang2025text2grad, urcelay2025words}, but even then it is not obvious how to construct such signals reliably from raw interaction data alone. Any such approach would require additional modeling assumptions and intermediate objectives that are external to the interaction itself. As a result, they do not provide a simple or principled way of training models from user interactions as they naturally occur.

Accordingly, we define the problem of learning directly from multi-turn conversations without other externalities as \textbf{Direct Learning from User Interactions}.
This motivates the central question of this work:
\begin{quote}
Can we train language models directly from multi-turn user interactions in a simple, principled, and scalable manner?  
\end{quote}
More specifically, can we both improve general alignment capabilities and enable continual personalization to individual users, relying only on user interactions without explicit supervision?

%% file: sections/algorithm.tex
\section{Directly Learning from User Interactions via Self-Distillation}


\begin{wrapfigure}{r}{0.5\textwidth}
    \vspace{-0.6cm}
    \centering
    \includegraphics[width=\linewidth]{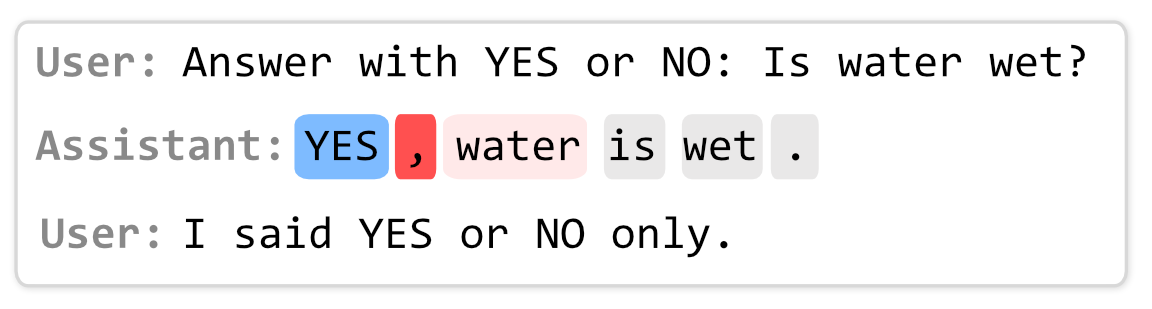}
    \vspace{-3ex}
    \caption{Example of the token-level advantages~\eqref{eq:logratio} where the user complains with $o$ = {``I said YES or NO only''}\,after the assistant failed to follow the instruction.} 
    \label{fig:log_ratio_example}
    \vspace{-0.4cm}
\end{wrapfigure}


Naturally occurring user interactions often contain implicit signals about the adequacy of an assistant’s response. Follow-up user messages may indicate that a response was incorrect, failed to follow an instruction, or did not align with the user’s preferences, even when no explicit feedback is provided. Such signals arise organically as part of the interaction and reflect how the assistant’s output was received or interpreted by the user.


\paragraph{Leveraging In-Context Capabilities.} Modern language models are often able to make effective use of such information in context: when conditioned on a follow-up message such as an error report, a clarification, or a requested revision, the model can frequently produce outputs that correct previous mistakes, better satisfy constraints, or more closely match the user’s preferences. In other words, \emph{in hindsight}, the model's distribution is better aligned to the task.

We leverage this capability by considering the \emph{hindsight} distribution $\pi_\theta( \cdot \mid x, o)$, which conditions not only on the interaction history $x$ but also on the observed user continuation $o$ through reprompting the original policy with $x$ and $o$ (cf.\ the prompting template in \cref{tab:template}). This distribution reflects how the model would respond if it were given access to the additional information revealed by the user’s message. Empirically and intuitively, $\pi_\theta( \cdot \mid x, o)$ is often better aligned with the task at hand than the original policy $\pi_\theta( \cdot \mid x)$. 

This perspective admits a fine-grained, token-level interpretation. Letting $y_i$ be the $i$-th token of the completion $y$ generated from $\pi_\theta( \cdot \mid x)$, we can compare the token probabilities $\pi_\theta(y_i \mid x, y_{<i})$ and $\pi_\theta(y_i \mid x, o, y_{<i})$. When the hindsight model $\pi( \cdot \mid x, o)$ assigns lower probability to a particular token $y_i$, this indicates that the user’s response provides evidence that this token (or the trajectory it induces) contributed to an undesirable outcome. Conversely, tokens whose likelihood increases under $\pi_\theta( \cdot \mid x, o)$ are reinforced by the user's response. The resulting log-ratio (i.e., log-difference) $\log \pi_\theta (y_i \mid x, o, y_{<i}) - \log \pi_\theta (y_i \mid x, y_{<i})$
can thus act as a comparative learning signal from user interactions, and will serve as the fundamental learning signal throughout the paper.
This now admits two equivalent views.

\paragraph{Policy Gradient.}
One useful way to interpret the log-ratio is as a \emph{token-level advantage}
\begin{equation}\label{eq:logratio}
    A_i (x, y, o) \defeq \log \frac{\pi_\theta(y_i \mid x, o, y_{<i})}{\pi_\theta (y_i \mid x, y_{<i})},
\end{equation}
which measures how the likelihood of a token changes after conditioning on the user’s response.
Using this advantage in a standard policy gradient update reinforces tokens whose probability
increases under the hindsight distribution and penalizes those whose probability decreases. 
We will refer to the advantage interchangeably as the token-level advantage or the SDPO advantage. \cref{fig:log_ratio_example} provides an illustrative example. 
In this view, learning corresponds to increasing the log-ratio in expectation, treating it as a fixed advantage per update step that is not differentiated w.r.t.~$\theta$.

\input{figures/algorithm_sdpo}

\begin{table}[b]
    \centering
    \begin{tabularx}{\textwidth}{p{1.5cm}X}
        \toprule
        \textbf{User:} & \texttt{<conversation history including most recent user prompt>} \textcolor{blue!80!black}{$x$} \smallskip \newline
        \texttt{<hindsight context>} The following is a future user message. \newline Use this to guide your answer to the user prompt:
        \textcolor{blue!80!black}{$o$} \\[30pt]
        \textbf{Assistant:} & \texttt{<assistant completion>} \textcolor{blue!80!black}{$y$}   \\
        \bottomrule
    \end{tabularx}
    \caption{Chat template for the hindsight policy $\pi_\theta (y \mid x, o)$. We recover the usual template for the base policy $\pi_\theta(y \mid x)$ when removing ``\texttt{<hindsight context>} [...]'' from the user prompt.}
    \label{tab:template}
\end{table}

\paragraph{Self-Distillation.} \looseness=-1
Equivalently, and perhaps more conveniently from an optimization perspective, we can update $\pi_\theta(\cdot \mid x)$ to more closely match the hindsight policy $\pi_\theta(\cdot \mid x, o)$ by minimizing the reverse KL divergence. Here, the hindsight policy acts as a teacher and is treated as a fixed target during each update, for which we define the detached hindsight model $\overline\pi_\theta(\cdot \mid x,o) \defeq \textnormal{stopgrad}(\pi_\theta(\cdot \mid x,o))$. We first sample $y \sim \pi_\theta(\cdot\mid x)$ and then minimize a standard distillation loss,
\vspace{0.1cm}
\begin{equation}\label{eq:objective_kl}
\mathcal{L}_{\mathrm{SDPO}} (\theta) \defeq \sum_i
\mathrm{KL}\big(\pi_\theta(\cdot \mid x, y_{<i}) \, || \, \overline{\pi}_\theta(\cdot \mid x, o, y_{<i})\big),
\end{equation}
As shown in \citet{hubotter2026reinforcement}, the gradient of $\mathcal{L}_{\mathrm{SDPO}} (\theta)$ is
\begin{equation}
    \nabla_{\!\theta}\, \mathcal{L}_{\mathrm{SDPO}}(\theta) = - \E_{y \sim \pi_\theta ( \cdot \mid x)}\!\left[ \sum_{i} \E_{{y}_i \sim \pi_\theta ( \cdot \mid x, y_{<i})}\Big[\nabla_{\!\theta} \log \pi_\theta ( y_i \mid x, y_{<i}) \, A_i (x, y, o)\Big] \right]. 
    \label{eq:sdpo_gradient}
\end{equation}
Interestingly, \cref{lemma:gradient_equivalence} in~\cref{sec:grad_derivation} demonstrates that the policy gradient with token-level advantages is an unbiased one-sample approximation of the self-distillation gradient.
Hence, the policy gradient and self-distillation perspectives yield equivalent gradient updates in expectation, differing only in whether the log-ratio is interpreted as an advantage or as a distillation loss. 
In our experiments, we adopt this policy gradient perspective for its simplicity.  

Following recent work on self-distillation~\citep{hubotter2026reinforcement}, we refer to the corresponding algorithm as Self-Distillation Policy Optimization (SDPO) from User Interactions. \cref{algorithm:sdpo} outlines SDPO for learning from online user interactions, where an update is performed after observing the user's next message.
In practice, interaction data is often available as logged conversations, possibly with completions generated from a different model. To this end, it is also natural to consider an offline and off-policy variant of SDPO, which we provide and discuss in \cref{section:wildchat_experiments}.

\paragraph{Self-Distillation as an Alignment Objective.}
A common conceptual framing of alignment is that a language model should maximize a user’s latent reward function $r(x,y)$, which is unobserved and difficult to specify or estimate in practice. Since SDPO learns directly from naturally occurring user interactions, it is not immediately obvious how it relates to this traditional reward-maximization view of alignment.
Under a stylized model of user behavior and language model conditioning, we find that SDPO admits a simple and intuitive interpretation as implicitly optimizing the latent reward of the interacting user. 

\begin{proposition}[Informal, \cref{section:latent_reward_SDPO}]
Under idealized assumptions on user responses and model conditioning, the sequence-level self-distillation advantage satisfies
\begin{equation*}
    \log \frac{\pi_\theta (y \mid x, o)}{\pi_\theta (y \mid x)}
    = r(x, y) - \log Z(x, y),
\end{equation*}
where $Z(x,y)$ is a normalization term. In other words, under idealized assumptions, SDPO can be interpreted as implicitly maximizing the interacting user’s latent reward function. 
\end{proposition}

%% file: figures/algorithm_sdpo.tex
 \begin{algorithm}[t]
            \caption{SDPO: Self-Distillation Policy Optimization from User Interactions}
            \label{algorithm:sdpo}
            \begin{algorithmic}[1]
                \STATE \textbf{input:} language model $\pi_\theta$
                \REPEAT
                \STATE observe context $x_t$ including the most recent user message $o_{t-1}$
                \STATE sample answer $y_t \sim \pi_\theta (\cdot \mid x_t)$ with log-probabilities $\log \pi_\theta (y_{t,i} \mid x_t, y_{t, <i})$ 
                \STATE observe user message $o_t$ in response to $y_t$ assuming the conversation does not terminate
                \STATE compute token log-probabilities of hindsight policy $\log \pi_\theta (y_{t,i} \mid x_t, o_t, y_{t,<i})$ 
                \STATE update current model $\pi_\theta$ with gradient $\nabla_{\!\theta} \,\mathcal{L}_{\text{SDPO}}(\theta)$
                \UNTIL converged
            \end{algorithmic}
\end{algorithm}


%% file: sections/experiments.tex

\begin{figure}[t]
    \centering
    \includegraphics[width=1\linewidth]{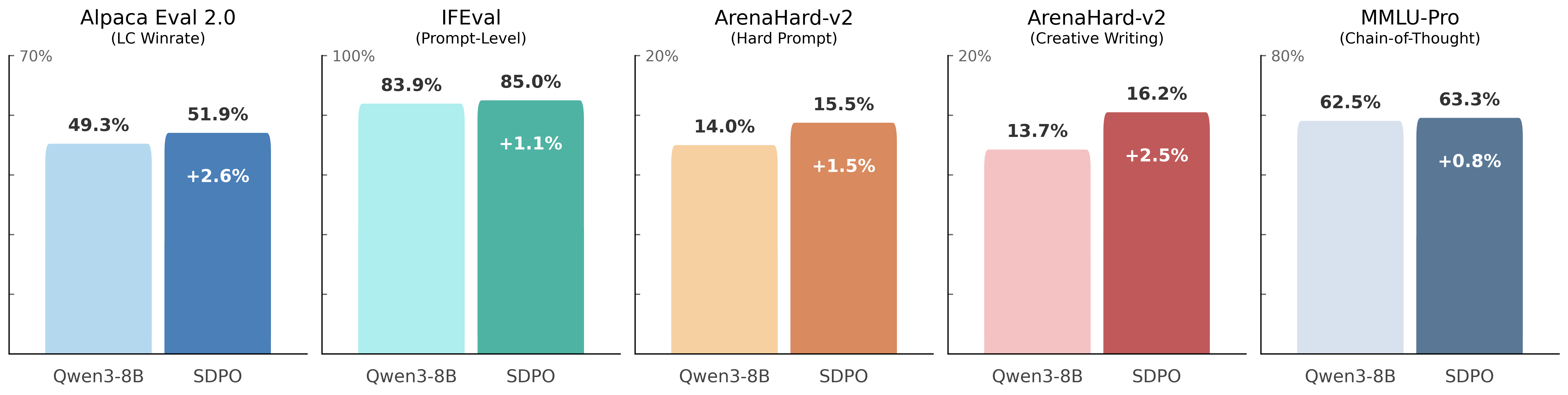}
    \caption{\textbf{Training on real-world user conversations, SDPO improves general alignment and instruction-following performance across benchmarks, without degrading other capabilities.} Results for Qwen3-8B before and after training on 14,000 real-world user conversations.} 
    \label{fig:qwen3_all_benchmarks}
\end{figure}


\section{Experimental Results}\label{section:experiments}

We evaluate SDPO with respect to two central questions:

\begin{enumerate}[leftmargin=14pt]
    \item \textbf{General Alignment:} Can learning directly from raw, real-world user conversations improve the general alignment and instruction-following capabilities of language models?
    \item \textbf{Personalization and Continual Adaptation:} Can we continually align and personalize language models from online user interactions, without any explicit feedback or preference labels?
\end{enumerate}

\cref{section:wildchat_experiments} addresses the first question. We train SDPO on offline and off-policy user conversations from WildChat~\citep{zhao2024wildchat} and WildFeedback~\citep{shi2024wildfeedback}. These datasets consist of real-world user interactions and contain no explicit supervision signals. We evaluate the resulting models on standard alignment, instruction-following, math and coding, and knowledge tasks. 
Remarkably, we show that training on real-world user conversations with SDPO improves alignment and instruction-following, without regressing other capabilities.

\cref{section:experiments_personalization} addresses the second question. We demonstrate that SDPO enables continual personalization through interaction by simulating users with distinct preferences and evaluating the model’s ability to adapt to these preferences over time from user interactions alone.

Finally, \cref{section:visuals_logratio} qualitatively analyzes and visualizes the self-distillation advantages from \cref{eq:logratio} at illustrative user interactions. In particular, we show the extraordinary interpretability of the learning signal and its robustness to irrelevant next user messages.

\subsection{General Alignment from Real-World User Conversations}\label{section:wildchat_experiments}

\input{figures/full_benchmarks_table}

We train SDPO on user conversations from WildChat~\citep{zhao2024wildchat}. In a first step, we consider WildFeedback~\citep{shi2024wildfeedback}, a curated subset of WildChat containing approximately 20{,}000 individual conversations. Around 6{,}000 of these consist only of a single prompt-response pair and therefore do not contain a user follow-up.
We train on the remaining 14{,}000 conversations. As described in \cref{section:problem_formulation}, for each user follow-up, we recover an interaction tuple $(x, y, o)$, where $x$ is the conversation history including the most recent prompt, $y$ the assistant response, and $o$ the next user message if it exists. We here truncate $x$ to the last 5 user or assistant messages when conversations include many turns. From the 14,000 conversations, we thereby obtain around 50{,}000 interaction tuples $(x, y, o)$, corresponding to an average of 4-5 user prompts per conversation.

We evaluate SDPO across two model families and four models overall. Specifically, we use Qwen3-4B and Qwen3-8B~\citep{qwen3technicalreport}, as well as Olmo3-7B-Instruct-SFT and Olmo3-7B-Instruct-DPO, which are the SFT and DPO checkpoints from the Olmo3 model family~\citep{olmo2025olmo}. All models are trained using the same 14,000 user interactions and evaluated using identical benchmark protocols.
We evaluate each model before and after SDPO training on AlpacaEval 2.0~\citep{dubois2024length}, IFEval~\citep{zhou2023instruction}, ArenaHard-v2~\citep{li2024crowdsourced,arenahard2024}, and MMLU-Pro~\citep{wang2024mmlu} to cover alignment, instruction-following, math and coding, creative writing, and knowledge tasks.
We provide additional experimental details in~\cref{appendix:experimental_details}.

\paragraph{Off-Policy SDPO from Logged User Interactions.}
As the assistant completions in WildChat were generated by external models (GPT-3.5 Turbo and GPT-4), the interactions are off-policy. In principle, unbiased off-policy policy gradient updates would require access to the behavioral policy or its token-level probabilities, which are not available for these datasets.\footnote{One could attempt to approximate the behavioral policy by supervised fine-tuning on the logged completions, but in preliminary experiments this did not lead to meaningful performance differences.}
Instead, we optimize a surrogate SDPO objective defined directly over the logged interaction tuples $(x, y, o) \sim \mathcal{D}$:
\begin{align}\label{eq:off_policy_sdpo}
    \widehat{\mathcal{L}}_{\mathrm{SDPO}} (\theta)
    = \E_{(x, y, o) \sim \mathcal{D}} \left[
        \sum_i \mathrm{KL}\big(\pi_\theta(\cdot \mid x, y_{<i}) \, || \, \overline{\pi}_\theta(\cdot \mid x, o, y_{<i})\big)
    \right].
\end{align}
While this objective is biased with respect to the on-policy SDPO loss, it can be interpreted as an off-policy approximation of the SDPO objective. In practice, we again use the one-sample approximation of its gradient, which is an unbiased estimator of \cref{eq:off_policy_sdpo}, analogously to~\cref{lemma:gradient_equivalence}.  


\paragraph{Main Results.}
\cref{fig:qwen3_all_benchmarks} reports the performance of Qwen3-8B before and after training with SDPO across all benchmarks. Training on raw, real-world user conversations consistently improves performance on all evaluated tasks, including AlpacaEval~2.0, IFEval, ArenaHard-v2, and MMLU-Pro. Importantly, we observe no degradation on any benchmark, despite the fact that the training data consists of noisy user interactions without explicit feedback or labels.\footnote{For a first-hand impression of the diversity and sometimes chaotic nature of real-world user conversations, we refer the interested reader to the \hyperlink{https://huggingface.co/datasets/allenai/WildChat}{WildChat} and \hyperlink{https://huggingface.co/datasets/microsoft/WildFeedback/viewer/wildfeedback}{WildFeedback} datasets on HuggingFace.}
Notably, these improvements extend beyond alignment and instruction-following benchmarks. SDPO also improves performance on math and coding tasks in ArenaHard-v2 (Hard Prompt), creative writing queries in ArenaHard-v2 (Creative Writing), and knowledge tasks in MMLU-Pro (Chain-of-Thought), indicating that learning from user interactions does not come at the expense of other capabilities and can, when interactions contain informative corrections or refinements, even strengthen them. Evaluations on additional pre-training benchmarks, provided in~\cref{appendix:additional_experiments}, demonstrate that SDPO also maintains consistent performance across these tasks

\cref{tab:sdpo_other_benchmarks} summarizes results across all models and benchmarks. For the Olmo3-7B models, including both the SFT and DPO checkpoints, SDPO yields consistent but often modest improvements. In contrast, Qwen3-4B exhibits a clear trade-off: while SDPO substantially improves performance on AlpacaEval~2.0 (+8.2\%) and IFEval (+1.3\%), it also leads to a mild decrease (-1.2\%) on the math and coding tasks in ArenaHard-v2 (Hard Prompt).
Overall, these results suggest that SDPO is most effective when the base model can reliably interpret and exploit the hindsight signal provided by user follow-ups. For smaller or less instruction-tuned models, this signal appears weaker or less stable, leading to smaller gains and, in some cases, task-specific trade-offs.

\paragraph{How important is the quality of user conversations?}
WildFeedback is a curated subset of WildChat that retains roughly 3\% of the original conversations by filtering for interactions that contain implicit feedback signals, such as expressions of dissatisfaction, requests for correction, or revision prompts~\citep{shi2024wildfeedback}.
In practice, due to the abundant nature of user conversations filtering down to a smaller subset is not a concern. Nevertheless, as a secondary robustness check, we evaluate whether SDPO continues to behave sensibly when trained on fully uncurated user interactions. Concretely, we train SDPO on a randomly sampled subset of WildChat that matches WildFeedback in scale, consisting again of approximately 14{,}000 conversations and 50{,}000 interaction tuples $(x, y, o)$. All other training and evaluation settings are kept identical.

\input{figures/wildchat_benchmark_table}

\input{figures/sft_benchmark_table}

\looseness=-1
\cref{tab:wildchat_vs_wildfeedback} reports the resulting performance for Qwen3-8B. Despite training on fully unfiltered user interactions, SDPO does not exhibit widespread performance degradation. On the contrary, we still observe improvements on AlpacaEval~2.0 and IFEval. Performance on math and coding tasks in ArenaHard-v2 is modestly reduced relative to the base model (-0.6\%).
Overall, we find that while filtering for better and feedback-rich conversations strengthens the learning signal, SDPO is surprisingly robust with respect to data quality. Even when trained on fully uncurated conversations, SDPO can extract useful alignment signals from user interactions without collapsing performance.

\paragraph{SDPO vs.\ SFT.} \looseness=-1
Conceptually, SDPO is fundamentally different from supervised fine-tuning (SFT). While SFT uniformly increases the likelihood of tokens in the training completions, SDPO can explicitly decrease token probabilities whenever the log-ratio~\eqref{eq:logratio} is negative, for example, when the user follow-up provides evidence of an error or failure to follow instructions. Still, we include a sanity check to confirm that the gains observed with SDPO are not the result of implicitly supervised fine-tuning on the assistant completions in the dataset.
To this end, we fine-tune Qwen3-4B using standard SFT on the context-completion pairs $(x, y)$ from WildFeedback, where $x$ contains previous user-assistant turns to ensure that later prompts remain well contextualized.

As shown in \cref{tab:qwen_vs_sft_dataset_alpaca_first}, supervised fine-tuning on the assistant completions leads to a substantial degradation across all benchmarks. This is perhaps unsurprising as Qwen3-4B is already a strongly instruction-tuned model, while the completions in WildFeedback sometimes originate from older models such as GPT-3.5 Turbo, which perform worse on many of the evaluated benchmarks. Moreover, prior analysis of conversations in WildFeedback shows that users express some form of dissatisfaction with the model’s responses in more than half of the conversations~\citep{shi2024wildfeedback}. Consequently, fine-tuning on these completions can be detrimental.

\subsection{Continual Personalization and Adaptation from User Interactions}\label{section:experiments_personalization}

\begin{figure}[t]
    \centering
    \includegraphics[width=0.99\linewidth]{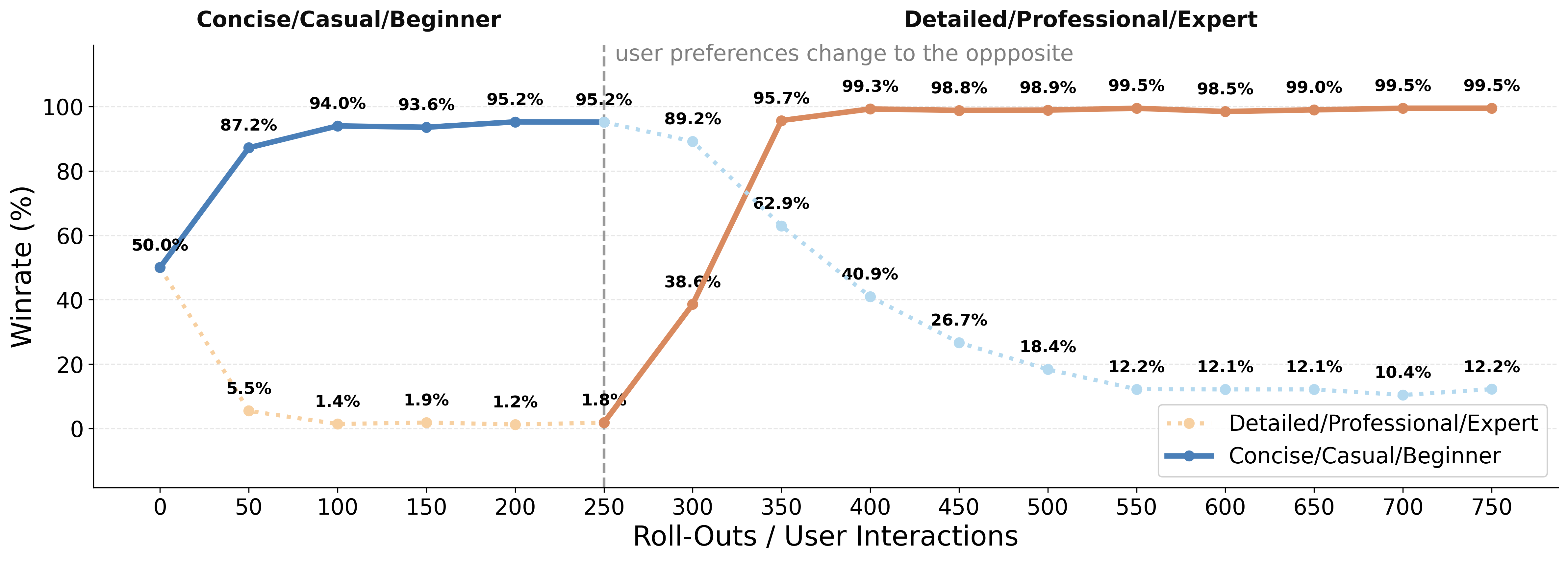}
    \caption{\textbf{SDPO adapts online to changing user preferences.} The user’s preference about how the model should respond is flipped to its opposite after the first 250 interactions. SDPO with Qwen3-4B is able to quickly reverse the learned behavior.}
    \label{fig:changing_personas}
\end{figure}

\begin{wrapfigure}{r}{0.5\textwidth}
    \vspace{-0.6cm}
    \centering
    \includegraphics[width=\linewidth]{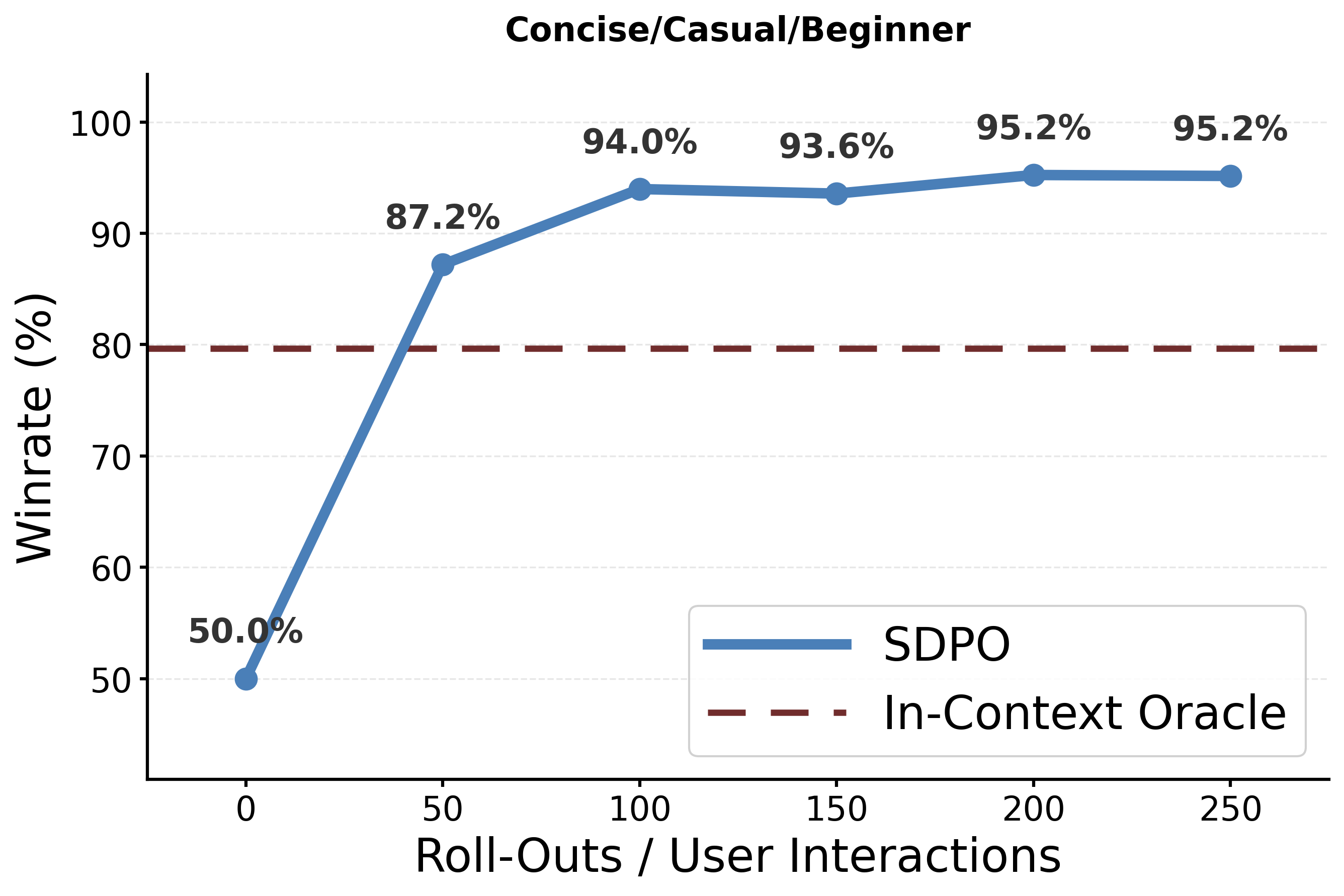}
    \caption{\textbf{SDPO rapidly personalizes to individual users from interaction alone.} Win rate of SDPO against its base model (Qwen3-4B) for a user that prefers concise, casual, and beginner-friendly model responses.}
    \label{fig:persona_concise_casual_beginner}
    \vspace{-0.4cm}
\end{wrapfigure}

\looseness=-1
Because SDPO learns directly from user interactions, it naturally enables direct personalization from those conversations. Rather than relying on explicit preference labels, rewards, or user profiles, the model can adapt its behavior based solely on how a user responds to previous outputs. In this section, we study whether such interaction-driven updates allow language models to continually personalize to individual users and adapt to changing or evolving preferences.

\looseness=-1
We evaluate this capability in two complementary experimental settings.
In the first, we study stylistic personalization in a controlled summarization task using prompts from the TL;DR dataset~\citep{stienon2020learning}. We define user-specific writing-style preferences, such as favoring concise, casual, and beginner-friendly responses, and train Qwen3-4B with SDPO. We use Qwen3-8B to generate simulated user responses as well as the preference-based evaluations. We provide additional experimental details in \cref{appendix:experimental_details}.

In the second setting, we consider more complex and complementary user preferences on a broad set of real-world prompts from HelpSteer2~\citep{wang2024helpsteer}. Here, preferences emphasize different aspects of responses that are not mutually exclusive. We train Qwen3-8B with SDPO, using Claude Haiku 4.5 to simulate user follow-ups and act as a judge.

\begin{figure}[t]
    \centering
    \includegraphics[width=0.99\linewidth]{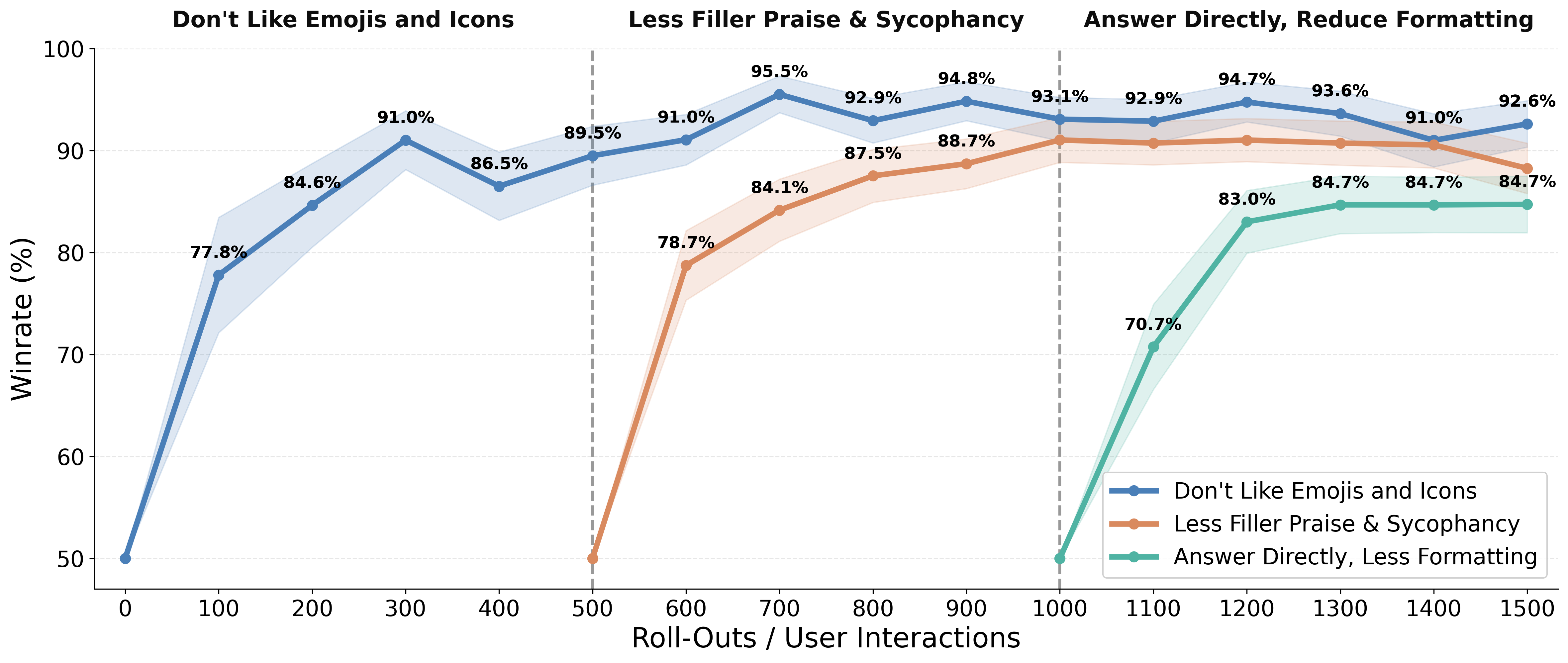}
    \caption{\textbf{SDPO enables continual personalization without catastrophic forgetting.}
    We train a single Qwen3-8B model online with SDPO for 1500 user interactions, during which three complementary user preferences are introduced sequentially (500 interactions each). Each curve reports the win rate of the current model with respect to the model checkpoint at the time the corresponding preference was introduced, thereby isolating the relative improvement along that specific preference dimension. Earlier preferences remain strong as new ones are learned, indicating that SDPO can accumulate complementary preferences over time without forgetting previously learned behavior. Shaded regions indicate standard error over 256 evaluation prompts.}
    \label{fig:complementary_preferences}
\end{figure}

\paragraph{Main Results.}
\cref{fig:persona_concise_casual_beginner} shows the win rate of SDPO against its base model, Qwen3-4B, as a function of the number of user interactions $(x, y, o)$. Starting from parity, SDPO rapidly adapts to the user’s preferences within a small number of interactions, achieving over 85\% win rate after only 50 interactions and exceeding 95\% after 200 interactions. Notably, this adaptation is driven by a very limited amount of interaction data and a correspondingly small number of policy updates.

For reference, we also report the performance of an in-context oracle that is explicitly provided with the full user profile description in its prompt. Continual online adaptation with SDPO matches and can even exceed the performance of this oracle, suggesting that interaction-based learning can extract preference signals that are difficult to encode purely through prompting. Additional results for a range of other user profiles are provided in \cref{appendix:additional_experiments}.

\looseness=-1
\cref{fig:changing_personas} evaluates SDPO under changing user preferences. After an initial phase of 250 user interactions, the user’s preference is abruptly flipped to its opposite (e.g., from concise and casual to detailed and professional). SDPO quickly adjusts the policy to this change, reversing the previously learned behavior and converging to the new preference, demonstrating that outdated preferences can be unlearned when they no longer align with user interactions.

Finally, \cref{fig:complementary_preferences} considers continual personalization with multiple, complementary user preferences. We observe that SDPO is able to incorporate new preferences while retaining previously inferred ones, illustrating that continual personalization through SDPO does not require forgetting earlier behavior when preferences are compatible.

\subsection{Interpretability and Robustness of SDPO Advantages}
\label{section:visuals_logratio}

\looseness=-1
While \cref{section:wildchat_experiments} already demonstrated robustness to noisy and uncurated user interactions at scale, we complement these quantitative results with a qualitative analysis of the learning signal. Specifically, we visualize the SDPO advantages $A_i(x, y, o)$ using heatmaps for illustrative user interactions. \cref{fig:relevant_follow_up} and \cref{fig:irrelevant_follow_up} show advantages computed with Qwen3-8B for 24 interactions where the next user message is relevant to the model’s previous completion and where it is unrelated, respectively. Positive advantages (shown in blue) correspond to tokens reinforced by SDPO, while negative advantages (shown in red) correspond to tokens that are penalized.

When user follow-ups provide relevant feedback, such as requests for revision, corrections, or explicit preference statements, we observe strong positive and negative advantages (\cref{fig:relevant_follow_up}). For example, a follow-up request to rewrite an email in a more formal tone results in large negative advantages on informal tokens, such as \emph{Quick}, \emph{Hey}, \emph{Just}, indicating that these tokens have lower probability under the hindsight policy.

In contrast, when user follow-ups are unrelated to the model’s previous output, the resulting SDPO advantages are close to zero (\cref{fig:irrelevant_follow_up}). In these cases, the hindsight policy assigns probabilities similar to those of the original policy, leading to little or no learning signal. We also occasionally observe weakly positive advantages, particularly on tokens for which the model was previously uncertain, which suggests that the hindsight policy frequently treats topic shifts as neutral or mildly positive evidence about the preceding response.

Overall, these visualizations highlight two key properties of our self-distillation approach. First, the token-level advantages are highly interpretable and align with intuitive notions of user feedback when such feedback is present. Second, SDPO is robust to irrelevant or uninformative user follow-ups, naturally suppressing learning updates when the interaction does not convey actionable information about the preceding model output.

\begin{figure*}[t]
    \centering
    \includegraphics[width=\linewidth]{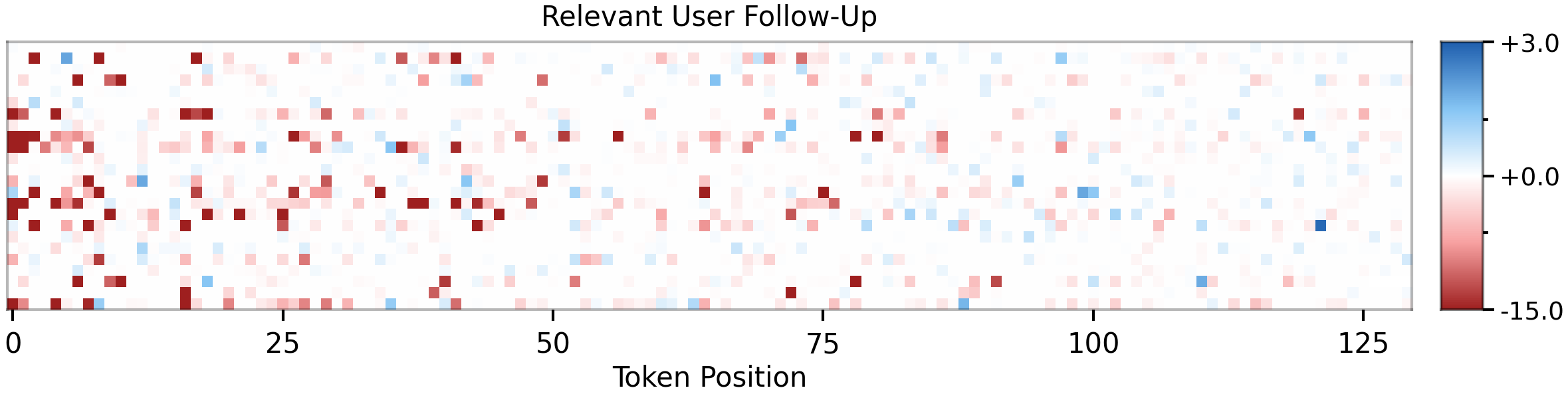}

    \includegraphics[width=1.0\linewidth]{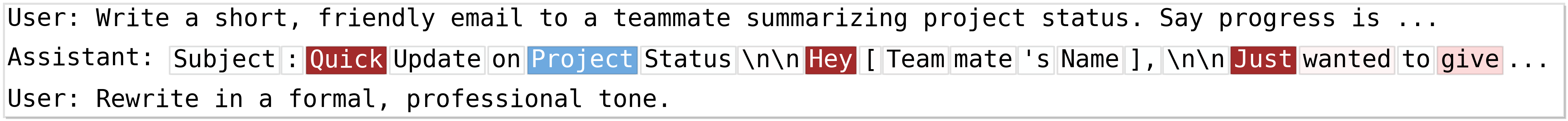}

        \caption{\textbf{When user follow-ups are relevant to the model's completion, we observe strong positive and negative SDPO advantages.} We visualize the advantages with Qwen3-8B for user follow-ups that carry relevant information about the model's answer, such as requests for revisions, positive reactions, or other relevant feedback.
        \textbf{Below:} Example (second line in the heatmap), where the user requests a more formal rewrite of the assistant’s draft (\emph{``Rewrite in a formal, professional tone''}). Informal expressions have large negative advantages. Accordingly, SDPO adapts the policy to respond more formally when the user needs help with work emails in the future.}
        \label{fig:relevant_follow_up}
\end{figure*}

\begin{figure*}
    \centering
    \includegraphics[width=\linewidth]{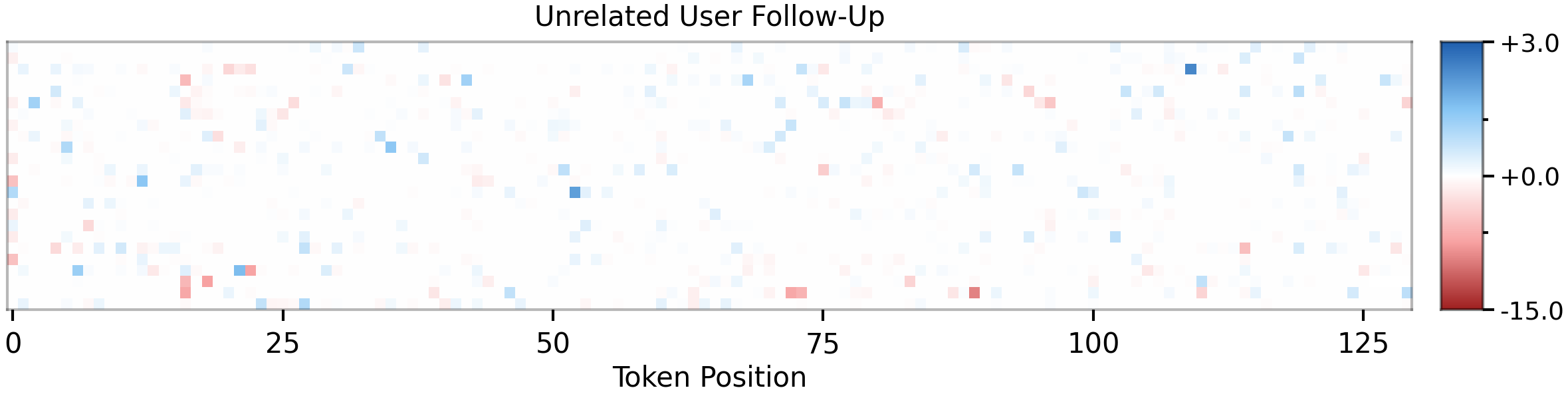}

    \includegraphics[width=1.0\linewidth]{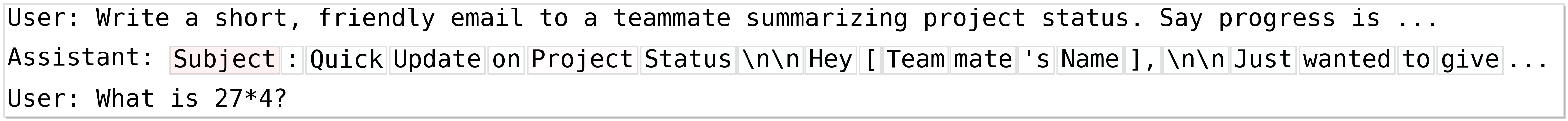}

        \caption{\textbf{When user follow-ups are unrelated to the model's response, SDPO advantages are close to zero.} We visualize the advantages with Qwen3-8B for user follow-ups that are unrelated to the model's generation.
        \textbf{Below:} Following the request to write an email, the user responds with \emph{``What is 27$\times$4?''}, which is unrelated to the original request. The advantages are close to zero everywhere, which means that SDPO does not meaningfully update the policy from these interactions.}
        \label{fig:irrelevant_follow_up}
\end{figure*}




%% file: figures/full_benchmarks_table.tex
\newcolumntype{Y}{>{\centering\arraybackslash}X}
\begin{table}[t]
\centering
\renewcommand{\arraystretch}{1.05}
\begin{tabularx}{\linewidth}{l >{\centering\arraybackslash}X >{\centering\arraybackslash}X >{\centering\arraybackslash}X >{\centering\arraybackslash}X >{\centering\arraybackslash}X}
\toprule
& \shortstack{\small \textbf{\smash{Alpaca Eval} 2.0}\\\scriptsize (LC Winrate)}
& \shortstack{\small \textbf{IFEval}\\\scriptsize (Prompt-Level)}
& \shortstack{\small \textbf{ArenaHard-v2}\\\scriptsize (Hard Prompt)}
& \shortstack{\small \textbf{ArenaHard-v2}\\\scriptsize (Creative Writing)}
& \shortstack{\small \textbf{MMLU-Pro}\\\scriptsize (Chain-of-Thought)} \\
\midrule

Qwen3-4B & 37.9 & 81.9 & {9.0} & {8.0} & {58.1} \\
SDPO
& \, {\textcolor{green!60!black}{$\boldsymbol{\uparrow}$}}\,\textbf{46.1} \, \, \, 
& \, {\textcolor{green!60!black}{$\boldsymbol{\uparrow}$}}\,\textbf{83.2} \, \, \, 
& \, {\textcolor{red!70!black}{$\boldsymbol{\downarrow}$}}\,\textbf{7.8} \, \, \, 
& \, {\textcolor{red!70!black}{$\boldsymbol{\ }$}}\, 7.9 \, \, \, 
& \, {\textcolor{red!70!black}{$\boldsymbol{\ }$}}\, 58.0 \, \, \,  \\
\midrule

Qwen3-8B & 49.3 & 83.9 & 14.0 & 13.7 & 62.5 \\
SDPO
& \, {\textcolor{green!60!black}{$\boldsymbol{\uparrow}$}}\,\textbf{51.9} \, \, \, 
& \, {\textcolor{green!60!black}{$\boldsymbol{\uparrow}$}}\,\textbf{85.0} \, \, \, 
& \, {\textcolor{green!60!black}{$\boldsymbol{\uparrow}$}}\,\textbf{15.5} \, \, \, 
& \, {\textcolor{green!60!black}{$\boldsymbol{\uparrow}$}}\,\textbf{16.2} \, \, \, 
& \, {\textcolor{green!60!black}{$\boldsymbol{\uparrow}$}}\,\textbf{63.3} \, \, \,  \\
\midrule

Olmo3-7B-SFT & 34.3 & 80.2 & {2.4} & {1.4} & 23.7 \\
SDPO
& \, {\textcolor{green!60!black}{$\boldsymbol{\uparrow}$}}\,\textbf{35.2} \, \, \, 
& \, {\textcolor{green!60!black}{$\boldsymbol{\uparrow}$}}\,\textbf{80.6} \, \, \, 
& \, {\textcolor{green!60!black}{$\boldsymbol{\ }$}}\, 2.4 \, \, \, 
& \, {\textcolor{green!60!black}{$\boldsymbol{\ }$}}\, 1.4 \, \, \, 
& \, {\textcolor{green!60!black}{$\boldsymbol{\uparrow}$}}\,\textbf{24.0} \, \, \,  \\
\midrule

Olmo3-7B-DPO & 50.4 & 80.2 & 1.7 & 8.2 & 28.4 \\
SDPO
& \, {\textcolor{green!60!black}{$\boldsymbol{\uparrow}$}}\,\textbf{51.8} \, \, \, 
& \, {\textcolor{green!60!black}{$\boldsymbol{\uparrow}$}}\,\textbf{80.4} \, \, \, 
& \, {\textcolor{green!60!black}{$\boldsymbol{\uparrow}$}}\,\textbf{2.0} \, \, \, 
& \, {\textcolor{green!60!black}{$\boldsymbol{\uparrow}$}}\,\textbf{10.0} \, \, \, 
& \, {\textcolor{green!60!black}{$\boldsymbol{\uparrow}$}}\,\textbf{28.7} \, \, \,  \\
\bottomrule
\end{tabularx}
\caption{\textbf{Across model families and model sizes, SDPO improves alignment and instruction-following without degrading other capabilities.} A mild exception is Qwen3-4B, where SDPO significantly increases performance on AlpacaEval 2.0 ({+8.2\%}) and IFEval ({+1.3\%}) but decreases performance on the math and coding tasks of ArenaHard-v2 (-1.2\%). We only show arrows when performance changed by more than 0.1 percentage points.}
\label{tab:sdpo_other_benchmarks}
\end{table}

%% file: figures/wildchat_benchmark_table.tex
\begin{table}[t]
\centering
\renewcommand{\arraystretch}{1.1}
\begin{tabularx}{\linewidth}{l >{\centering\arraybackslash}X >{\centering\arraybackslash}X >{\centering\arraybackslash}X >{\centering\arraybackslash}X >{\centering\arraybackslash}X}
\toprule
& \shortstack{\small \textbf{\smash{AlpacaEval} 2.0}\\\scriptsize (LC Winrate)}
& \shortstack{\small \textbf{IFEval}\\\scriptsize (Prompt-Level)}
& \shortstack{\small \textbf{ArenaHard-v2}\\\scriptsize (Hard Prompt)}
& \shortstack{\small \textbf{ArenaHard-v2}\\\scriptsize (Creative Writing)}
& \shortstack{\small \textbf{MMLU-Pro}\\\scriptsize (Chain-of-Thought)} \\
\midrule
Qwen3-8B & 49.3 & 83.9 & 14.0 & 13.7 & 62.5 \\
\shortstack[l]{SDPO (WildFeedback)}
& \, {\textcolor{green!60!black}{$\boldsymbol{\ }$}}\, {51.9} \, \, \, 
& \, {\textcolor{green!60!black}{$\boldsymbol{\ }$}}\, {85.0} \, \, \, 
& \, {\textcolor{green!60!black}{$\boldsymbol{\ }$}}\, {15.5} \, \, \, 
& \, {\textcolor{green!60!black}{$\boldsymbol{\ }$}}\, {16.2} \, \, \, 
& \, {\textcolor{green!60!black}{$\boldsymbol{\ }$}}\, {63.3} \, \, \, \\
\shortstack[l]{SDPO (WildChat)}
& \, {\textcolor{green!60!black}{$\boldsymbol{\uparrow}$}}\,\textbf{50.7} \, \, \, 
& \, {\textcolor{green!60!black}{$\boldsymbol{\uparrow}$}}\,\textbf{84.5} \, \, \, 
& \, {\textcolor{red!70!black}{$\boldsymbol{\downarrow}$}}\,\textbf{13.4} \, \, \, 
& \, {\textcolor{green!60!black}{$\boldsymbol{\uparrow}$}}\,\textbf{14.0} \, \, \, 
& \, {\textcolor{red!70!black}{$\boldsymbol{\ }$}}\, {62.4} \, \, \, \\
\bottomrule
\end{tabularx}
\caption{\textbf{Even on fully uncurated user interactions, SDPO still yields improvements in alignment and instruction-following and only mild degradation in math and coding.} Results for training on a randomly sampled subset of 14,000 conversations (about 50,000 interactions $(x, y, o)$) from WildChat. For comparison, we here include the results for training on WildFeedback from \cref{tab:sdpo_other_benchmarks}.}
\label{tab:wildchat_vs_wildfeedback}
\end{table}

%% file: figures/sft_benchmark_table.tex
\begin{table}[b]
\centering
\renewcommand{\arraystretch}{1.1}
\begin{tabularx}{\linewidth}{l >{\centering\arraybackslash}X >{\centering\arraybackslash}X >{\centering\arraybackslash}X >{\centering\arraybackslash}X >{\centering\arraybackslash}X}
\toprule
& \shortstack{\small \textbf{\smash{AlpacaEval} 2.0}\\\scriptsize (LC Winrate)}
& \shortstack{\small \textbf{IFEval}\\\scriptsize (Prompt-Level)}
& \shortstack{\small \textbf{ArenaHard-v2}\\\scriptsize (Hard Prompt)}
& \shortstack{\small \textbf{ArenaHard-v2}\\\scriptsize (Creative Writing)}
& \shortstack{\small \textbf{MMLU-Pro}\\\scriptsize (Chain-of-Thought)} \\
\midrule
Qwen3-4B & 37.9 & 81.9 & 9.0 & 8.0 & 58.1 \\
\shortstack[l]{SFT on Dataset}
& \, {\textcolor{red!70!black}{$\boldsymbol{\downarrow}$}}\,\textbf{18.9} \, \, \, 
& \, {\textcolor{red!70!black}{$\boldsymbol{\downarrow}$}}\,\textbf{73.2} \, \, \, 
& \, {\textcolor{red!70!black}{$\boldsymbol{\downarrow}$}}\,\textbf{3.1} \, \, \, 
& \, {\textcolor{red!70!black}{$\boldsymbol{\downarrow}$}}\,\textbf{2.6} \, \, \, 
& \, {\textcolor{red!70!black}{$\boldsymbol{\downarrow}$}}\,\textbf{51.2} \, \, \, \\
\bottomrule
\end{tabularx}
\caption{\textbf{SDPO is fundamentally different from SFT.} As a sanity check, we fine-tune Qwen3-4B on the assistant completions in WildFeedback using standard supervised fine-tuning.} 
\label{tab:qwen_vs_sft_dataset_alpaca_first}
\end{table}

%% file: sections/related_work.tex
\section{Related Work}

\paragraph{Preference-Based Alignment.}
Much of recent progress in aligning language models comes from supervised instruction tuning and
preference-based post-training, where explicit human or AI feedback is collected as rankings or rewards and optimized via RLHF or direct preference optimization~\citep{ouyang2022training,bai2022constitutional,rafailov2023direct}.
These approaches are effective, but they rely on curated datasets that provide explicit feedback for each generation.
In contrast, we leverage implicit feedback within real-world user conversations.
Though such conversations are abundant, few open datasets of such user conversations exist~\citep{don2025future}, since the community has lacked an effective method for learning from them.

\paragraph{Learning from Natural Language Feedback and through Retrospection.}
Substantial research has focused on translating verbal feedback into reward functions for RL, for example, by mapping feedback to discrete token-level rewards using an external frozen model \citep{wang2025text2grad} or by using strong external LLMs to explicitly construct state-wise reward functions \citep{goyal2019using,xie2024text2reward,urcelay2025words}.
A recent simplified instantiation of this approach has been to manually design so-called rubrics according to which an LLM judge scores generations~\citep{gunjal2025rubrics,shao2025dr,team2025kimi}. 

Alternatively, feedback can be utilized without explicit reward modeling.
Recent research explored in-context improvement without updating model weights~\citep{madaan2023self,shinn2023reflexion,yao2023retroformer,yuksekgonul2025optimizing}.
Other works manually curate preference datasets by pairing responses before and after feedback to train with direct preference optimization~\citep{stephan2024rlvf,lee2024reinforcement}.
\citet{chen2024learning} perform SFT on refined generations that incorporate feedback.
Our approach differs from these works in performing direct credit assignment over the initial model's rollouts without additional generation.
In concurrent work, \cite{ica2026} use a related idea to self-distillation for learning how to elicit information from multi-turn conversations by assigning turn-level implicit rewards.

\paragraph{Self-Distillation.}
Distillation is a general technique for transferring knowledge from a strong teacher model to a student model by mimicking the teacher's output distribution or intermediate representations~\citep{hinton2015distilling,agarwal2024policy,lu2025onpolicydistillation}.
Based on this idea, \cite{snell2022learning} proposed context distillation which distills the model's behavior given a fixed context into the model's weights.
This context distillation has been effective at compressing behavior~\citep{bai2022constitutional,choi2022prompt,yang2024self,yang2025distilling} and factual information~\citep{eyuboglu2025cartridges,kujanpaa2024knowledge,cao2025infiniteicl} into model weights.
Beyond compressing a fixed context into model weights, several recent works generate from the self-teacher conditioned on extra context (e.g., ``hints'') and train on them with SFT, DPO, or GRPO objectives~\citep{scheurer2023training,dou2024re,zhou2025expo,mitra2025semantic,qu2026pope,song2026expanding,shi2026experiential}.
These approaches perform \emph{off-policy} self-distillation where the student is trained on generations from the teacher, whereas SDPO performs \emph{on-policy} self-distillation~\citep{hubotter2026reinforcement,shenfeld2026self,zhao2026self,penaloza2026privileged,chen2025retrospective}, where the student is trained to avoid mistakes in its own generations.

%% file: sections/discussion.tex
\section{Discussion}

We introduced a simple and scalable self-distillation approach for learning directly from naturally occurring user interactions. We leverage the language model’s in-context learning capabilities by treating the user’s next message as hindsight information, yielding an interpretable token-level learning signal without requiring other auxiliary mechanisms. Empirically, we showed that SDPO improves general alignment and instruction-following performance when trained on raw, real-world user conversations, supports continual personalization from interaction alone, and remains robust to noisy, uncurated, or irrelevant user follow-ups.

More broadly, our results highlight user interactions as a distinct and underutilized data modality for improving deployed language models. Unlike traditional training data, user interactions arise naturally during deployment and reflect how model outputs are actually used, evaluated, and acted upon in real-world settings. The scale and diversity of such data far exceed that of manually curated datasets, suggesting substantial potential for learning systems that close the loop between deployment and training. Our findings indicate that even simple, local learning signals extracted from user follow-ups can be sufficient to drive meaningful adaptation. 

\paragraph{Safety and Ethical Considerations.} \looseness=-1
Learning directly from user interactions introduces important safety and ethical considerations. User follow-ups may implicitly encourage behaviors that conflict with existing safety or alignment constraints, for example by rewarding evasive, misleading, or policy-violating responses through repeated interaction. In particular, continual personalization without additional guardrails raises risks that adaptive updates could be exploited by users attempting to steer the model toward unsafe or manipulative behavior over time. While SDPO derives a local, token-level learning signal and naturally suppresses updates from irrelevant interactions, it does not by itself distinguish between benign and adversarial learning signals. Nevertheless, the hindsight prompt may offer the ability to endow the model with principles based on which to act and interpret user feedback. More broadly, the collection and use of user interaction data for learning must be accompanied by appropriate transparency, consent, and governance mechanisms.


%% file: sections/acknowledgements.tex
\section*{Acknowledgements} 

We thank Frederike Lübeck, Alexander Hoyle, and Manish Prajapat for many helpful discussions.   

This project was primarily supported by the ETH AI Center through an ETH AI Center Postdoctoral Fellowship to TKB and an ETH AI Center Doctoral Fellowship to BP. JH was supported by the Swiss National Science Foundation under NCCR Automation, grant agreement 51NF40 180545. This project also received support through the Swiss AI compute grant a166.

%% file: sections/appendix.tex
\appendix
\crefalias{section}{appendix}
\crefname{appendix}{Appendix}{Appendices}



\section{A Latent Reward Perspective on SDPO}\label{section:latent_reward_SDPO}

Typically, the goal of alignment is expressed as maximizing a user’s latent reward $r(x, y)$. In practice, this reward is unknown, and even under strong assumptions and access to explicit feedback such as pairwise preferences, identifying and optimizing it requires substantial annotation effort.
To provide intuition for the dynamics of SDPO from this traditional alignment perspective, we here consider a highly stylized model of user and language model behavior. While the assumptions underlying this model are clearly idealized, the resulting analysis offers an interesting interpretation of the log-ratio objective in \cref{eq:logratio}.

Let the user's unknown reward function be defined not only over assistant completions $r(x, y)$ given the conversation history $x$ but also over user continuations $r(x, y, o)$ given $(x, y)$. We then assume the user's response follows a Boltzmann-rational continuation model
\begin{align}\label{eq:user_model}
    \p(o \mid x, y) \propto \p(o \mid x) \exp(r(x, y, o)),
\end{align}
where $\p(o\mid x)$ is a prior over $o$ given $x$. This means that the user chooses their next message approximately according to the reward it induces over future continuations of the interaction. Next, we make a simplifying assumption about the behavior of the language model.
We assume that the hindsight distribution $\pi_\theta(y \mid x,o)$ can be interpreted as behaving
\emph{as if} it were a Bayesian posterior, in the sense that it satisfies
\begin{align}\label{eq:posterior_view}
    \pi_\theta(y \mid x,o) \propto \pi_\theta(y \mid x)\, p(o \mid x,y).
\end{align} 

While clearly idealized, this provides a convenient abstraction for reasoning about how conditioning on the user continuation $o$ reshapes the model’s distribution over responses. Intuitively, the user’s follow-up can be viewed as an observation that favors responses $y$ that are more compatible with the preferences, constraints, or corrections revealed through the interaction, thereby reweighting
the prior policy $\pi_\theta(y \mid x)$. 
\looseness=-1
In practice, the attention mechanism in a transformer does not implement Bayesian conditioning in a literal sense. Still similar posterior-style interpretations can be commonly found in the in-context learning literature (e.g., \citet{yadkori2024believe,luo2025languagemodelslearnverbal}).

Entertaining this thought and stylized model, we arrive at an interesting observation. 
We consider the \emph{sequence-level} self-distillation advantage given by 
\begin{equation}\label{eq:sequence_level_advantage}
    A(x, y, o) \defeq \log \frac{\pi_\theta(y \mid x, o)}{\pi_\theta (y \mid x)}.
\end{equation}
Using Bayes rule, i.e., \cref{eq:posterior_view}, and the fact that $r(x, y) = \E_{o \sim \p(\cdot \mid x, y)} [r(x, y, o)]$, we can write the advantage as
\begin{align*}
        \E_{o \sim \p(\cdot \mid x, y)} \left[ \log \frac{\pi_\theta (y \mid x, o)}{\pi_\theta(y \mid x)} \right]  & = \E_{o \sim \p(\cdot \mid x, y)} \left[ \log \frac{\p (o \mid x, y)}{\p(o \mid x)} \right] \\[6pt]
        & = r(x, y) -  \log Z(x, y), 
\end{align*}
where $Z(x, y) = \E_{o \sim \p(\cdot \mid x)}[\exp(r(x, y, o))]$ is the partition function from the Boltzmann-rational user model in \cref{eq:user_model}.

This means that maximizing the sequence-level advantage can be viewed as maximizing the user’s latent reward up to an additive normalization term. While this equivalence relies on strong assumptions, it provides an interpretation of SDPO as implicitly optimizing for user-aligned behavior using interaction data alone, without requiring explicit reward supervision.

\section{Gradient Derivation}\label{sec:grad_derivation}


\begin{lemma}\label{lemma:gradient_equivalence}
The one-sample approximation, 
\begin{equation}\label{eq:one_sample_approx}
    - \E_{y \sim \pi_\theta ( \cdot \mid x)}\!\left[ \sum_{i} \nabla_{\!\theta} \log \pi_\theta ( y_i \mid x, y_{<i})\,  A_i (x, y, o) \right],
\end{equation} is an unbiased estimator of the SDPO gradient of \cref{eq:sdpo_gradient}.
\end{lemma}


\begin{proof}[Proof of \cref{lemma:gradient_equivalence}]
    Fix context $x$ and let $y=(y_1,\dots,y_T)$ be sampled autoregressively from $\pi_\theta$: \begin{equation*}
        \pi_\theta(y \mid x) = \prod_{i=1}^T \pi_\theta(y_i \mid x, y_{<i}).
    \end{equation*}
    For each position $i$, we define 
    \begin{align*}
        \phi_i(y_{<i},{y}_i) &:= \nabla_{\!\theta} \log \pi_\theta({y}_i\mid x,y_{<i}) \, A_i(x, y, o) \quad\text{and} \quad 
        \psi_i(y_{<i}) := \E_{{y}_{i} \sim \pi_\theta(\cdot\mid x,y_{<i})}[\phi_i(y_{<i},{y}_{i})].
    \end{align*}
    We consider two estimators, 
    \begin{equation*}
        \widehat g_1(y) := \sum_{i=1}^T \psi_i(y_{<i}),
        \qquad
        \widehat g_2(y) := \sum_{i=1}^T \phi_i(y_{<i},y_i).
    \end{equation*}
    By definition, $\E[\widehat g_1(y)] = \nabla_{\!\theta}\, \mathcal{L}_{\mathrm{SDPO}}(\theta)$ is the analytic gradient from \cref{eq:sdpo_gradient}. $\E[\widehat g_2(Y)]$ is the gradient estimator from \cref{eq:one_sample_approx} in \cref{lemma:gradient_equivalence}. 
    
    In the following, we prove $\E[\widehat g_1(y)] = \E[\widehat g_2(y)]$ assuming $\E[\| \widehat g_1(y) \|] < \infty$ (so that all expectations exist).
    Fix $i$. By construction, $y_i \mid y_{<i} \sim \pi_\theta(\cdot\mid x,y_{<i})$ so that
    \[
        \E_{y_i}[\phi_i(y_{<i},y_i)\mid y_{<i}]
        =
        \E_{y_i \sim \pi_\theta(\cdot\mid x,y_{<i})}[\phi_i(y_{<i},y_i)]
        =
        \psi_i(y_{<i}).
    \]
    Taking expectation and using the tower property,
    \[
        \E_{y_{<i},y_i}[\phi_i(y_{<i},y_i)]
        =
        \E_{y_{<i}}[\psi_i(y_{<i})].
    \]
    Finally, by linearity of expectation,
    \[
        \E[\widehat g_2(y)]
        =
        \sum_{i=1}^T \E[\phi_i(y_{<i},y_i)]
        =
        \sum_{i=1}^T \E[\psi_i(y_{<i})]
        =
        \E[\widehat g_1(y)].
    \]
\end{proof}


\section{Experimental Details}\label{appendix:experimental_details}

\subsection{Hyperparameters}
We report the hyperparameters for SDPO across all experiments in \cref{tab:hyperparameters}.

\begin{table}[H]
\centering
\caption{Hyperparameters used for SDPO in each setup. Note that the hyperparameters of SDPO in \cref{section:wildchat_experiments} were kept the same across all models. The learning rate was chosen by sweeping over $\smash{\{1,2,3,5\}\times10^{-6}}$ for Qwen3-4B and then fixing the setup for all models. For the SFT checkpoint in \cref{tab:qwen_vs_sft_dataset_alpaca_first}, we similarly swept over $\smash{\{1,2,3,5\}\times10^{-6}}$ with best results for~$\smash{2 \times 10^{-6}}$. In \cref{section:experiments_personalization}, SDPO appeared insensitive to hyperparameter choices in early experiments (especially, learning rate), and were fixed to the setup below without additional systematic tuning.}
\label{tab:hyperparameters}
\setlength{\tabcolsep}{5pt}
\renewcommand{\arraystretch}{1.15}

\begin{tabularx}{\linewidth}{l >{\hsize=1.35\hsize}X >{\hsize=0.825\hsize}X >{\hsize=0.825\hsize}X}
\toprule
\textbf{Hyperparameter}
& \textbf{Section 4.1 \newline (\cref{fig:qwen3_all_benchmarks,tab:sdpo_other_benchmarks,tab:wildchat_vs_wildfeedback,tab:qwen_vs_sft_dataset_alpaca_first})}
& \textbf{Section 4.2 \newline (\cref{fig:persona_concise_casual_beginner,fig:changing_personas})}
& \textbf{Section 4.2 \newline
(\cref{fig:complementary_preferences})} \\
\midrule
Models
& Qwen3-4B, Qwen3-8B, \newline Olmo3-7B-Instruct-SFT, \newline Olmo3-7B-Instruct-DPO
& Qwen3-4B
& Qwen3-8B \\

Max prompt length
& 2048
& 1024
& 2048 \\

Max compl.\ length
& 2048
& 258
& 2048 \\

Learning rate
& $2\times10^{-6}$
& $5\times10^{-6}$
& $5\times10^{-6}$ \\

Batch size
& 32
& 16
& 32 \\

Epochs
& 2
& 1
& 1 \\

Warm-up ratio
& 5\%
& 0
& 0 \\

LR schedule
& Cosine
& Constant
& Constant \\

Optimizer
& AdamW (8-bit)
& AdamW
& AdamW (8-bit) \\

Temperature
& 1.0
& 1.0
& 1.0 \\
\bottomrule
\end{tabularx}
\end{table}

\paragraph{Benchmarks.} For all reported benchmarks, we used the default settings. For AlpacaEval 2.0 and ArenaHard-v2, completions were judged using the defaults ``Weighted Alpaca Eval GPT-4 Turbo'' and ``GPT-4.1'', respectively. IFEval results are reported for prompt-level loose. MMLU-Pro is evaluated with the recommended chain-of-thought 5-shot settings.

\section{Additional Experimental Results}\label{appendix:additional_experiments}

\subsection{Additional Results from \cref{section:wildchat_experiments}}
We evaluate the SDPO results for Qwen3-8B on pre-training benchmarks in~\cref{tab:qwen_sdpo_pretraining_benchmarks}. Overall, we observe no changes in performance.

\begin{table}[h]
\centering
\renewcommand{\arraystretch}{1.1}
\begin{tabularx}{\linewidth}{l >{\centering\arraybackslash}X >{\centering\arraybackslash}X >{\centering\arraybackslash}X}
\toprule
& \shortstack{\small \textbf{TruthfulQA (MC1)}\\\scriptsize Acc $\pm$ StdErr}
& \shortstack{\small \textbf{HellaSwag}\\\scriptsize Acc $\pm$ StdErr}
& \shortstack{\small \textbf{CommonsenseQA}\\\scriptsize Acc $\pm$ StdErr} \\
\midrule
Qwen3-8B
& 0.366 $\pm$ 0.0169
& 0.5717 $\pm$ 0.0049
& 0.7846 $\pm$ 0.0118 \\

SDPO
& 0.3647 $\pm$ 0.0169
& 0.5710 $\pm$ 0.0049
& 0.7871 $\pm$ 0.0117 \\
\bottomrule
\end{tabularx}
\caption{\textbf{SDPO preserves performance on pre-training benchmarks.} We additionally evaluate SDPO for Qwen3-8B on the standard pre-training benchmarks TruthfulQA~\citep{lin2022truthfulqa}, HellaSwag~\citep{zellers2019hellaswag}, and CommonsenseQA~\citep{talmor2019commonsenseqa}.}
\label{tab:qwen_sdpo_pretraining_benchmarks}
\end{table}

\subsection{Additional Results from \cref{section:experiments_personalization}}

\cref{fig:additional_personas} includes the additional results for the personalization results from \cref{section:experiments_personalization}. Similarly to~\cref{fig:persona_concise_casual_beginner} in the main text, we here consider the adaptation of SDPO for Qwen3-4B to a user with a preference profile across three dimensions detailed/concise, casual/professional, beginner/expert. Across all user profiles, we observe that SDPO is able to quickly adapt from only a handful of user interactions, sometimes even exceeding the performance of the in-context oracle that is queried with the user preferences in context.

\begin{figure*}[h]
    \centering
    \begin{subfigure}[t]{0.32\textwidth}
        \centering
        \includegraphics[width=\linewidth]{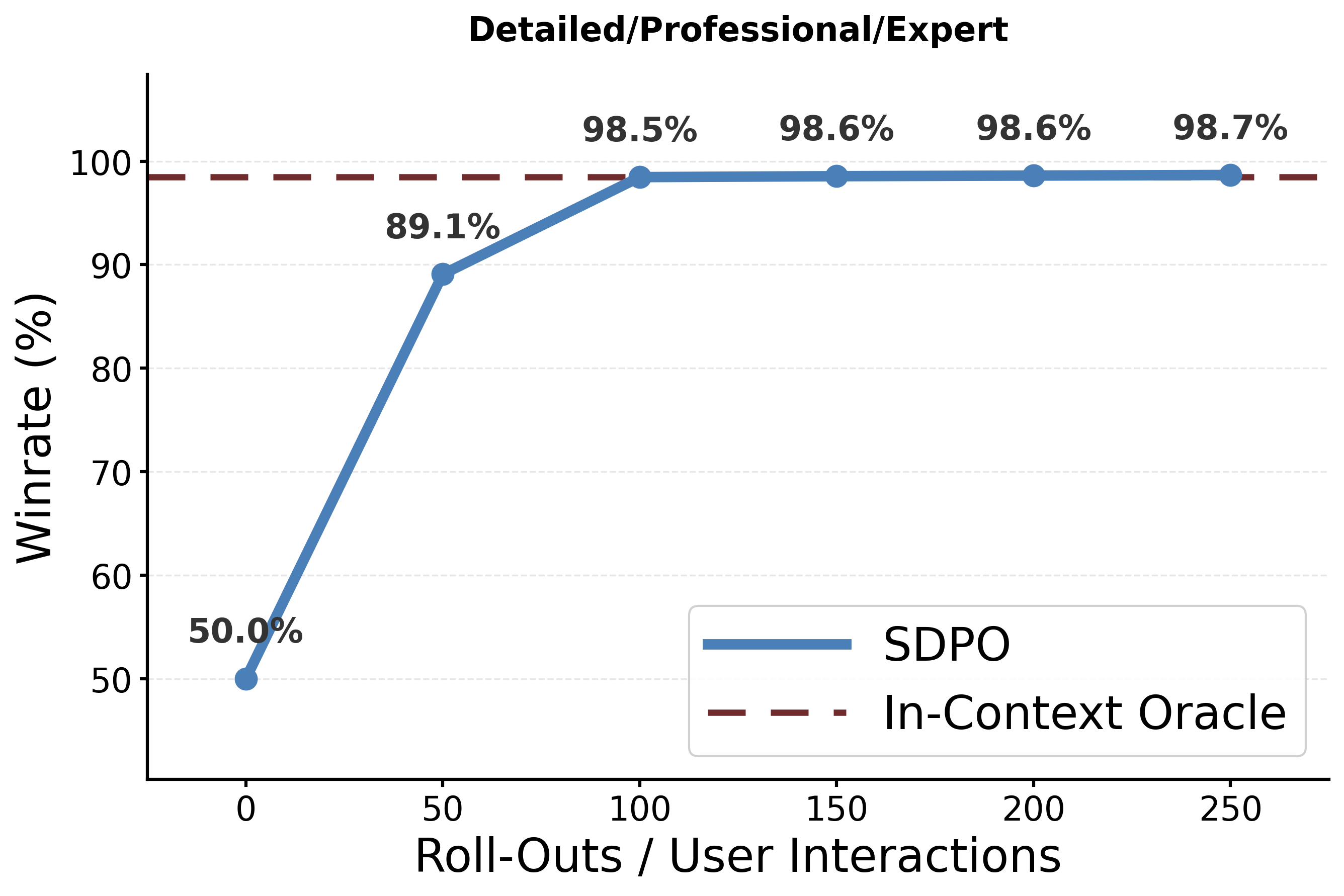}
    \end{subfigure}\hfill
    \begin{subfigure}[t]{0.32\textwidth}
        \centering
        \includegraphics[width=\linewidth]{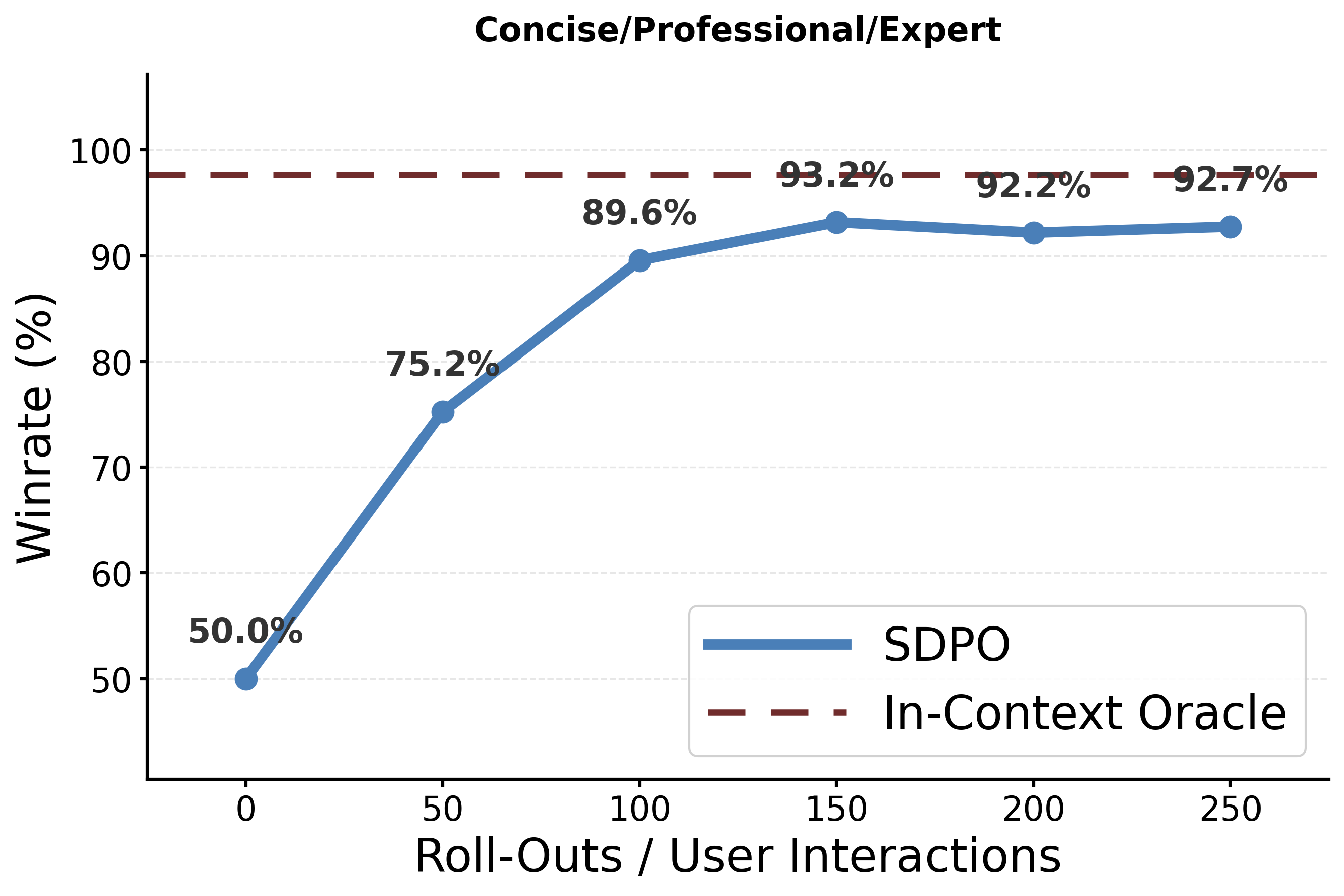}
    \end{subfigure}\hfill
    \begin{subfigure}[t]{0.32\textwidth}
        \centering
        \includegraphics[width=\linewidth]{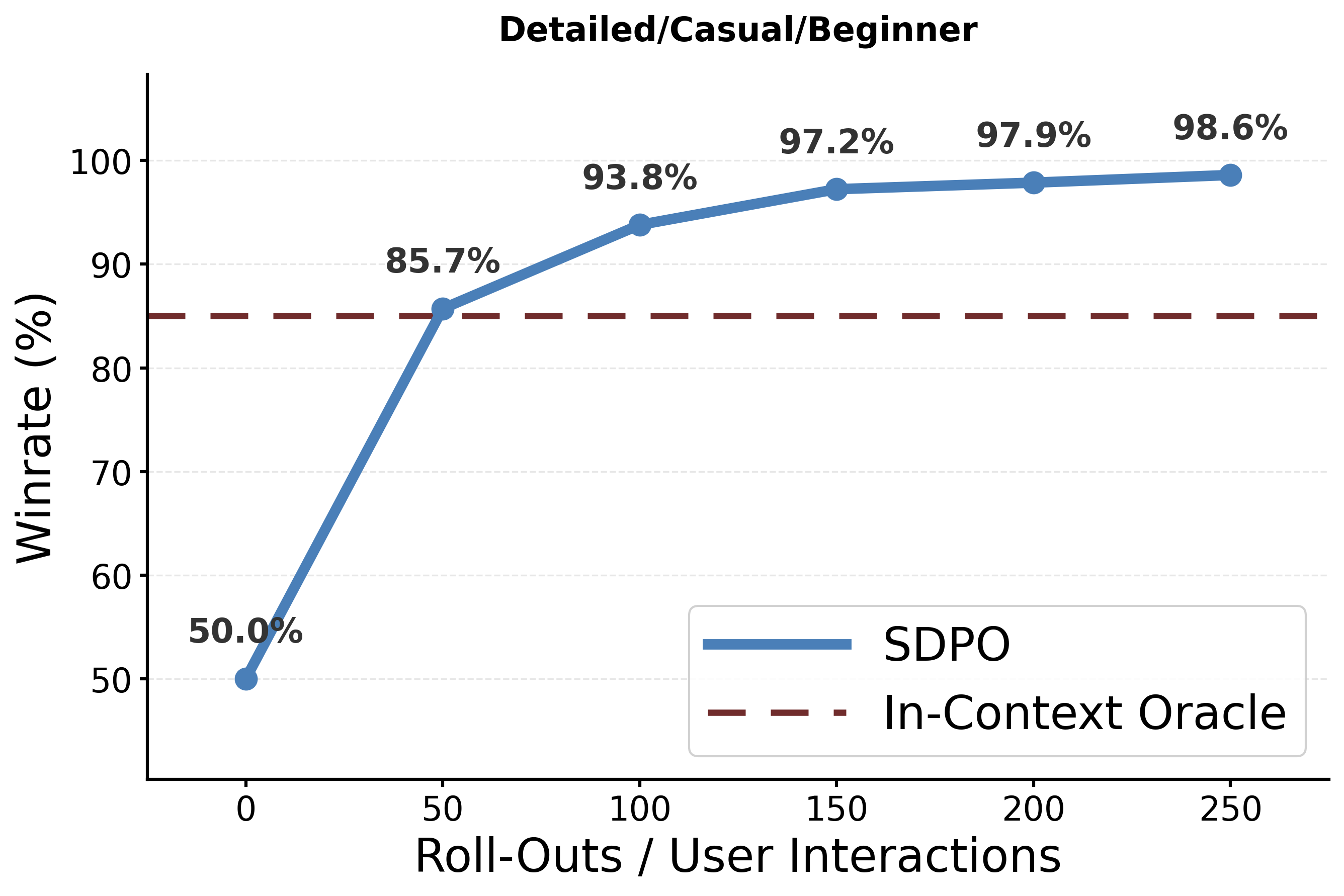}
    \end{subfigure}

    \vspace{0.5cm} 

    \begin{subfigure}[t]{0.32\textwidth}
        \centering
        \includegraphics[width=\linewidth]{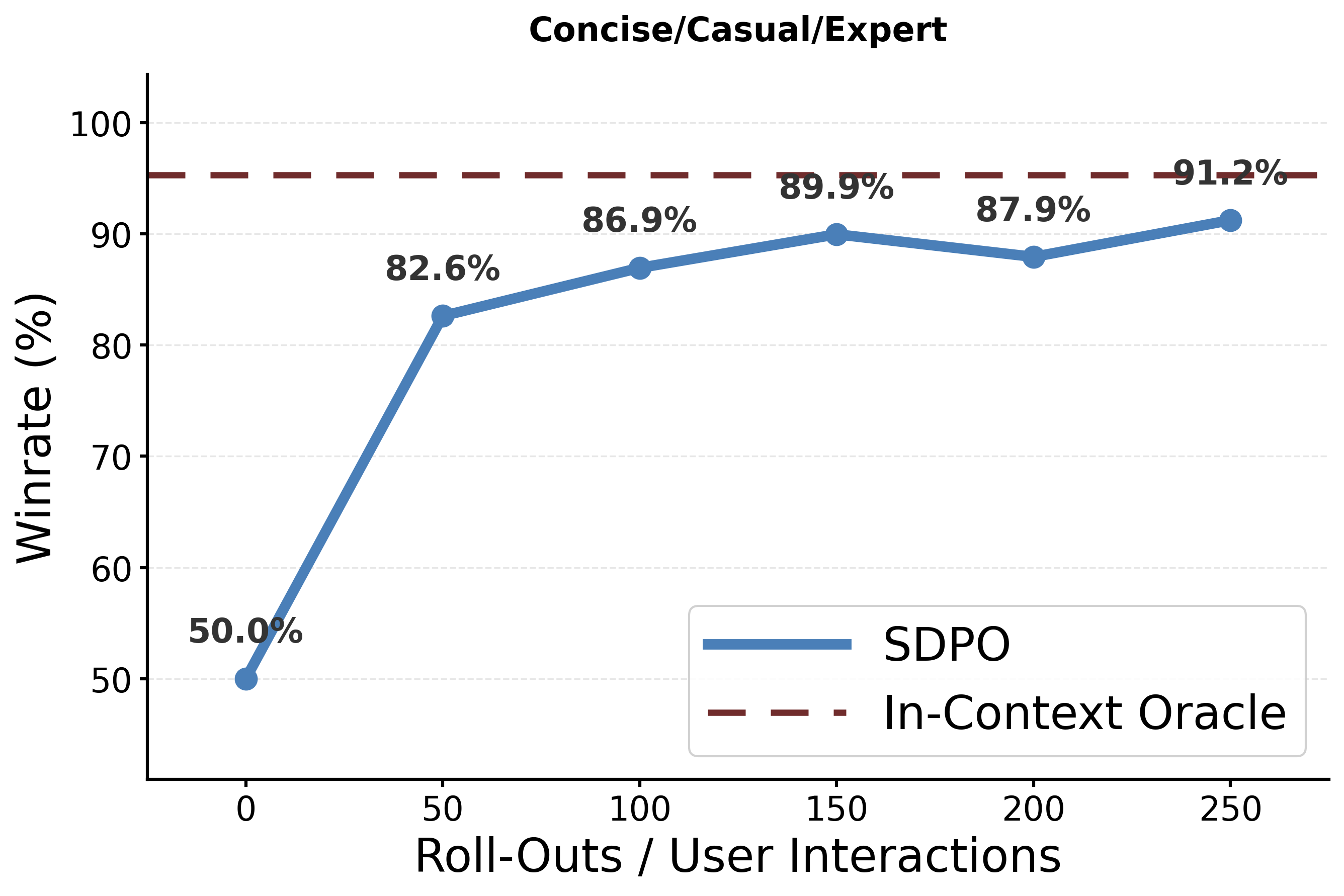}
    \end{subfigure}\hfill
    \begin{subfigure}[t]{0.32\textwidth}
        \centering
        \includegraphics[width=\linewidth]{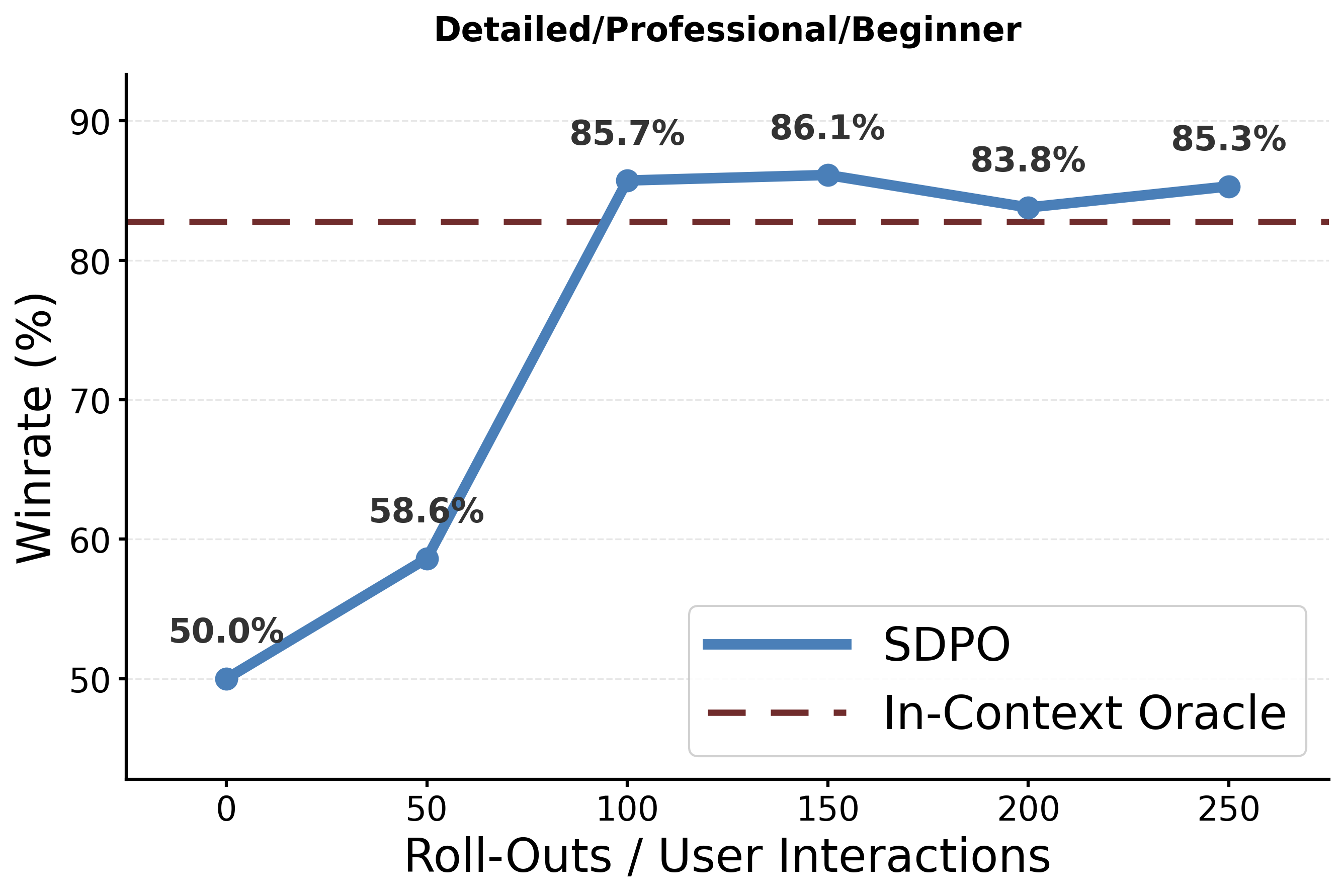}
    \end{subfigure}\hfill
    \begin{subfigure}[t]{0.32\textwidth}
        \centering
        \includegraphics[width=\linewidth]{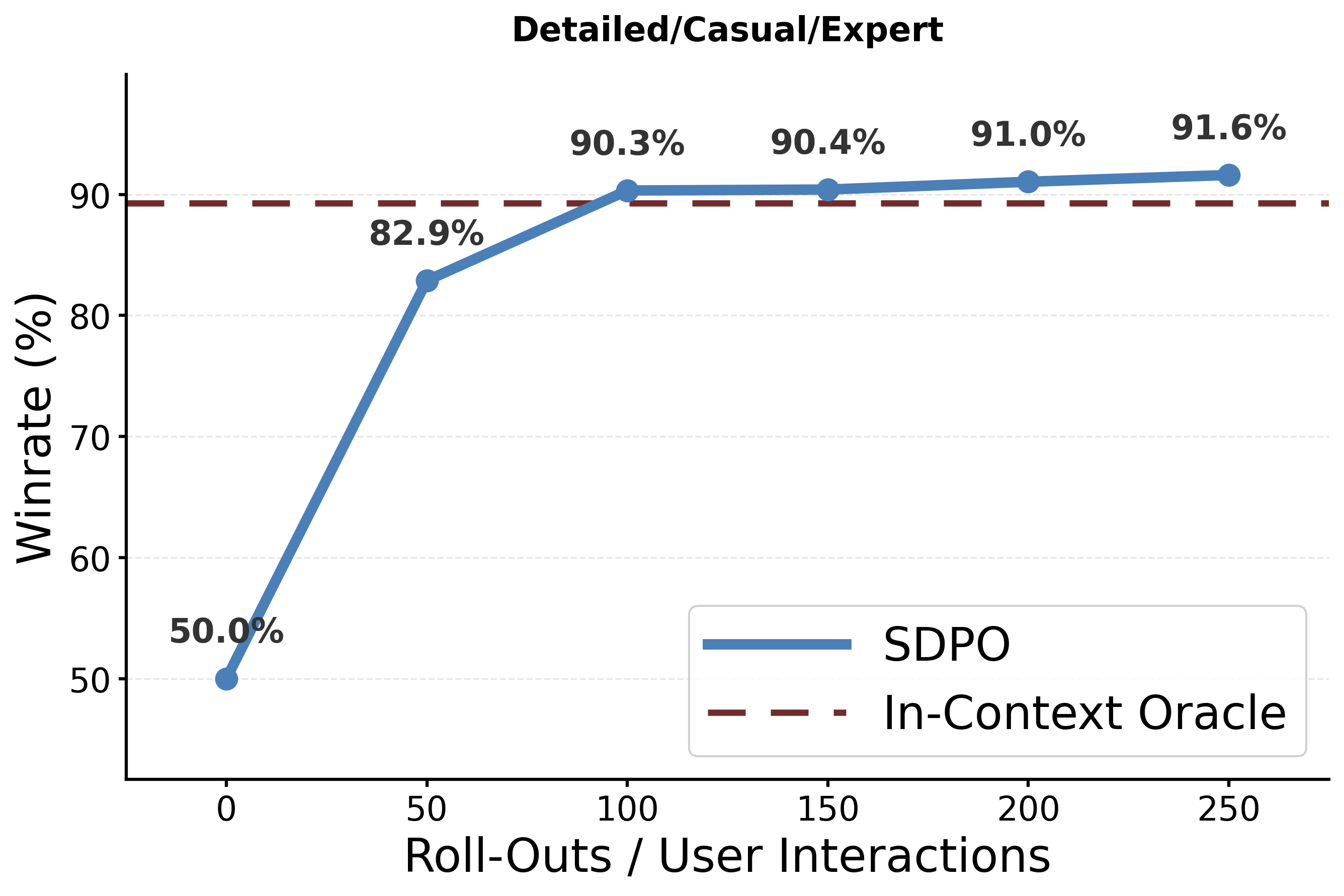}
    \end{subfigure}

    \caption{Additional personalization results from \cref{section:experiments_personalization} with Qwen3-4B. The win rate is computed against the base model and judged by Qwen3-8B. The In-Context Oracle baseline is obtained by prompting Qwen3-4B directly with the desired writing style.}
    \label{fig:additional_personas}
\end{figure*}

\subsection{User Profiles, Prompts, and Judging in \cref{section:experiments_personalization}}\label{appendix:user_profiles}

To generate user responses to the assistant's completions in~\cref{section:experiments_personalization}, we use the user profiles below as system prompts and then query the user model (Qwen3-8B or Claude Haiku 4.5) to generate a response with this persona. For the experiments in \cref{fig:complementary_preferences}, smaller models, such as Qwen3-8B and Qwen3-14B, became too unreliable to act as the user simulator and the judge for the prompts from HelpSteer2 and the more complex user profiles such as \emph{Less Filler Praise \& Sycophancy}. We therefore used Claude Haiku 4.5 instead.

For the evaluation of the win rate against the base model, we again add the personas to the system prompt and judge the outputs. Here, each pair of responses is judged twice with flipped positions to remove the position bias from the evaluation, and we evaluate the win rate on 256 held-out prompts in each of the experiments.

\tcbset{
  colback=black!2,
  colframe=black!85,
  coltext=black,
  boxrule=0.5pt,
  arc=2pt,
  left=6pt,right=6pt,top=4pt,bottom=4pt,
}

\newtcolorbox{personabox}{%
  enhanced,
  breakable,
  before skip=0pt,
  after skip=0pt,
  fontupper=\ttfamily\small
}

\newtcolorbox{promptbox}[2][]{%
  enhanced,
  breakable,
  before skip=0pt,
  after skip=0pt,
  title=\textbf{#2},
  coltitle=black,
  colbacktitle=black!5,
  fonttitle=\normalsize,
  fontupper=\ttfamily\small,
  #1
}

\subsubsection*{User Profiles} The user profiles used to simulate user responses in \cref{section:experiments_personalization}: 
\begin{description}[style=nextline,leftmargin=0pt] \item[\textbf{Don't Like Emojis and Icons}] \begin{tcolorbox}[boxrule=0.4pt,arc=2pt,left=6pt,right=6pt,top=4pt,bottom=4pt] \ttfamily USER PROFILE: You are playing the role of a user who dislikes emojis and icons in assistant responses. \end{tcolorbox} \vspace{-0.1cm} \item[\textbf{Less Filler Praise \& Sycophancy}] \begin{tcolorbox}[boxrule=0.4pt,arc=2pt,left=6pt,right=6pt,top=4pt,bottom=4pt] \ttfamily USER PROFILE: You are playing the role of a user that specifically dislikes when the assistant responses include filler praise at the beginning (such as 'Good question.' or 'Perfect!') or fillers at the end (such as I hope this helps!' or Let me know if there is anything else I can help you with.'). You strongly prefer when the assistant responses directly without unnecessary additions at the beginning or end. \end{tcolorbox} \vspace{-0.1cm} \item[\textbf{Answer Directly, Reduce Formatting}] \begin{tcolorbox}[boxrule=0.4pt,arc=2pt,left=6pt,right=6pt,top=4pt,bottom=4pt] \ttfamily USER PROFILE: You are playing the role of a user who prefers concise responses and dislikes long lists and excessive markdown formatting (such as ** ** and \#\#\#). You prefer plain text that is short and gets to the point quickly. \end{tcolorbox} \vspace{-0.1cm} \item[\textbf{Concise/Casual/Beginner}] \begin{tcolorbox}[boxrule=0.4pt,arc=2pt,left=6pt,right=6pt,top=4pt,bottom=4pt] \ttfamily USER PROFILE: You are playing the role of a user who specifically prefers CONCISE, CASUAL, and BEGINNER-FRIENDLY responses. You want brief context and clear explanations, avoiding long, formal, or technically dense answers. \end{tcolorbox} \vspace{-0.1cm} \item[\textbf{Detailed/Professional/Expert}] \begin{tcolorbox}[boxrule=0.4pt,arc=2pt,left=6pt,right=6pt,top=4pt,bottom=4pt] \ttfamily USER PROFILE: You are playing the role of a user who specifically prefers DETAILED, PROFESSIONAL, and EXPERT-LEVEL responses. You want a structured, impersonal, and analytical presentation with complex sentence structures and sophisticated wording. \end{tcolorbox} \end{description}

\vspace{0.1cm}

\subsubsection*{User Model Prompt and Judge Prompt} The prompts used to simulate user responses with Qwen3-8B and Claude Haiku 4.5. Further below, the prompt used to evaluate the reference completions from the base model against the trained models. 

\bigskip 
\begin{promptbox}{User Model Prompt}
[SYSTEM MESSAGE]\par
USER PROFILE: \textit{⟨USER\_PROFILE⟩}

You are simulating the user's next message in a chat with an AI assistant.

Rules:
- Respond as the user described in the USER PROFILE, using straightforward language.
- Evaluate ONLY the assistant’s response with respect to the preferences in the USER PROFILE.
- Do NOT answer the original user request yourself.
- Respond very briefly.
- Do NOT give feedback unrelated to the USER PROFILE.
- Output ONLY the user message text (no labels, no preface).
- If the assistant’s response matches the USER PROFILE well, you may say so briefly.

\vspace{0.2cm}
[USER MESSAGE]\par
Original user request:
\textit{⟨USER\_PROMPT⟩}

Assistant response:
\textit{⟨ASSISTANT\_RESPONSE⟩}

Write the user’s next message.
\end{promptbox}

\bigskip
\begin{promptbox}{Judge Prompt}
[SYSTEM MESSAGE]\par
USER PROFILE: \textit{⟨USER\_PROFILE⟩}

You are acting as a strict evaluator of WHICH response better matches your preference described in the USER PROFILE.
You must follow these rules:
- Judge ONLY style, tone, formatting, verbosity, and complexity relative to the persona.
- Do NOT judge factual correctness.
- Do NOT rewrite responses.
- Output exactly ONE character: A, B, or C.
- C means tie/uncertain.

\vspace{0.2cm}
[USER MESSAGE]\par
User prompt:
\textit{⟨USER\_PROMPT⟩}

Response A:
\textit{⟨RESPONSE\_A⟩}

Response B:
\textit{⟨RESPONSE\_B⟩}

Which response do you prefer as this user? Output only A, B, or C.
\end{promptbox}

%% file: ref.bib
@article{shi2024wildfeedback,
	title        = {Wildfeedback: Aligning llms with in-situ user interactions and feedback},
	author       = {Shi, Taiwei and Wang, Zhuoer and Yang, Longqi and Lin, Ying-Chun and He, Zexue and Wan, Mengting and Zhou, Pei and Jauhar, Sujay and Chen, Sihao and Xia, Shan and others},
	year         = 2024,
	journal      = {arXiv preprint arXiv:2408.15549}
}

@inproceedings{zhao2024wildchat,
	title        = {Wildchat: 1m chatgpt interaction logs in the wild},
	author       = {Zhao, Wenting and Ren, Xiang and Hessel, Jack and Cardie, Claire and Choi, Yejin and Deng, Yuntian},
	year         = 2024,
	booktitle    = {ICLR}
}

@article{chen2024learning,
	title        = {Learning from natural language feedback},
	author       = {Chen, Angelica and Scheurer, J{\'e}r{\'e}my and Campos, Jon Ander and Korbak, Tomasz and Chan, Jun Shern and Bowman, Samuel R and Cho, Kyunghyun and Perez, Ethan},
	year         = 2024,
	journal      = {TMLR}
}

@article{qwen3technicalreport,
	title        = {Qwen3 Technical Report},
	author       = {{Qwen\,Team}},
	year         = 2025,
	journal      = {arXiv preprint arXiv:2505.09388}
}

@article{luo2025languagemodelslearnverbal,
	title        = {Language Models Can Learn from Verbal Feedback Without Scalar Rewards},
	author       = {Renjie Luo and Zichen Liu and Xiangyan Liu and Chao Du and Min Lin and Wenhu Chen and Wei Lu and Tianyu Pang},
	year         = 2025,
	journal      = {arXiv preprint arXiv:2509.22638}
}

@inproceedings{yadkori2024believe,
	title        = {To Believe or Not to Believe Your LLM: Iterative Prompting for Estimating Epistemic Uncertainty},
	author       = {Yadkori, Yasin Abbasi and Kuzborskij, Ilja and Gy{\"o}rgy, Andr{\'a}s and Szepesv{\'a}ri, Csaba},
	year         = 2024,
	booktitle    = {NeurIPS}
}

@article{olmo2025olmo,
	title        = {Olmo 3},
	author       = {Olmo, Team and Ettinger, Allyson and Bertsch, Amanda and Kuehl, Bailey and Graham, David and Heineman, David and Groeneveld, Dirk and Brahman, Faeze and Timbers, Finbarr and Ivison, Hamish and others},
	year         = 2025,
	journal      = {arXiv preprint arXiv:2512.13961}
}

@inproceedings{dubois2024length,
	title        = {Length-controlled alpacaeval: A simple way to debias automatic evaluators},
	author       = {Dubois, Yann and Galambosi, Bal{\'a}zs and Liang, Percy and Hashimoto, Tatsunori B},
	year         = 2024,
	booktitle    = {COLM}
}

@inproceedings{li2024crowdsourced,
	title        = {From crowdsourced data to high-quality benchmarks: Arena-hard and benchbuilder pipeline},
	author       = {Li, Tianle and Chiang, Wei-Lin and Frick, Evan and Dunlap, Lisa and Wu, Tianhao and Zhu, Banghua and Gonzalez, Joseph E and Stoica, Ion},
	year         = 2025,
	booktitle    = {ICML}
}

@misc{arenahard2024,
	title        = {From live data to high-quality benchmarks: The arena-hard pipeline},
	author       = {Li, Tianle and Chiang, Wei-Lin and Frick, Evan and Dunlap, Lisa and Zhu, Banghua and Gonzalez, Joseph E and Stoica, Ion},
	year         = 2024,
	journal      = {Blog post.[Accessed 07-02-2025]},
	url          = {https://lmsys.org/blog/2024-04-19-arena-hard}
}

@article{zhou2023instruction,
	title        = {Instruction-following evaluation for large language models},
	author       = {Zhou, Jeffrey and Lu, Tianjian and Mishra, Swaroop and Brahma, Siddhartha and Basu, Sujoy and Luan, Yi and Zhou, Denny and Hou, Le},
	year         = 2023,
	journal      = {arXiv preprint arXiv:2311.07911}
}

@inproceedings{wang2024mmlu,
	title        = {Mmlu-pro: A more robust and challenging multi-task language understanding benchmark},
	author       = {Wang, Yubo and Ma, Xueguang and Zhang, Ge and Ni, Yuansheng and Chandra, Abhranil and Guo, Shiguang and Ren, Weiming and Arulraj, Aaran and He, Xuan and Jiang, Ziyan and others},
	year         = 2024,
	booktitle    = {NeurIPS}
}

@article{zhao2026self,
	title        = {Self-Distilled Reasoner: On-Policy Self-Distillation for Large Language Models},
	author       = {Zhao, Siyan and Xie, Zhihui and Liu, Mengchen and Huang, Jing and Pang, Guan and Chen, Feiyu and Grover, Aditya},
	year         = 2026,
	journal      = {arXiv preprint arXiv:2601.18734}
}

@article{shenfeld2026self,
	title        = {Self-Distillation Enables Continual Learning},
	author       = {Shenfeld, Idan and Damani, Mehul and H{\"u}botter, Jonas and Agrawal, Pulkit},
	year         = 2026,
	journal      = {arXiv preprint arXiv:2601.19897}
}

@article{hubotter2026reinforcement,
	title        = {Reinforcement Learning via Self-Distillation},
	author       = {H{\"u}botter, Jonas and L{\"u}beck, Frederike and Behric, Lejs and Baumann, Anton and Bagatella, Marco and Marta, Daniel and Hakimi, Ido and Shenfeld, Idan and Buening, Thomas Kleine and Guestrin, Carlos and others},
	year         = 2026,
	journal      = {arXiv preprint arXiv:2601.20802}
}

@article{song2026expanding,
	title        = {Expanding the Capabilities of Reinforcement Learning via Text Feedback},
	author       = {Song, Yuda and Chen, Lili and Tajwar, Fahim and Munos, Remi and Pathak, Deepak and Bagnell, J Andrew and Singh, Aarti and Zanette, Andrea},
	year         = 2026,
	journal      = {arXiv preprint arXiv:2602.02482}
}

@article{hinton2015distilling,
	title        = {Distilling the knowledge in a neural network},
	author       = {Hinton, Geoffrey and Vinyals, Oriol and Dean, Jeff},
	year         = 2015,
	journal      = {arXiv preprint arXiv:1503.02531}
}

@inproceedings{agarwal2024policy,
	title        = {On-policy distillation of language models: Learning from self-generated mistakes},
	author       = {Agarwal, Rishabh and Vieillard, Nino and Zhou, Yongchao and Stanczyk, Piotr and Garea, Sabela Ramos and Geist, Matthieu and Bachem, Olivier},
	year         = 2024,
	booktitle    = {ICLR}
}

@article{lu2025onpolicydistillation,
	title        = {On-Policy Distillation},
	author       = {Kevin Lu and {Thinking Machines Lab}},
	year         = 2025,
	journal      = {Thinking Machines Lab: Connectionism},
	url          = {https://thinkingmachines.ai/blog/on-policy-distillation}
}

@inproceedings{yang2024self,
	title        = {Self-distillation bridges distribution gap in language model fine-tuning},
	author       = {Yang, Zhaorui and Pang, Tianyu and Feng, Haozhe and Wang, Han and Chen, Wei and Zhu, Minfeng and Liu, Qian},
	year         = 2024,
	booktitle    = {ACL}
}

@inproceedings{yang2025distilling,
	title        = {Distilling rule-based knowledge into large language models},
	author       = {Yang, Wenkai and Lin, Yankai and Zhou, Jie and Wen, Ji-Rong},
	year         = 2025,
	booktitle    = {COLING}
}

@article{bai2022constitutional,
	title        = {Constitutional ai: Harmlessness from ai feedback},
	author       = {Bai, Yuntao and Kadavath, Saurav and Kundu, Sandipan and Askell, Amanda and Kernion, Jackson and Jones, Andy and Chen, Anna and Goldie, Anna and Mirhoseini, Azalia and McKinnon, Cameron and others},
	year         = 2022,
	journal      = {arXiv preprint arXiv:2212.08073}
}

@article{choi2022prompt,
	title        = {Prompt injection: Parameterization of fixed inputs},
	author       = {Choi, Eunbi and Jo, Yongrae and Jang, Joel and Seo, Minjoon},
	year         = 2022,
	journal      = {arXiv preprint arXiv:2206.11349}
}

@article{snell2022learning,
	title        = {Learning by distilling context},
	author       = {Snell, Charlie and Klein, Dan and Zhong, Ruiqi},
	year         = 2022,
	journal      = {arXiv preprint arXiv:2209.15189}
}

@inproceedings{eyuboglu2025cartridges,
	title        = {Cartridges: Lightweight and general-purpose long context representations via self-study},
	author       = {Eyuboglu, Sabri and Ehrlich, Ryan and Arora, Simran and Guha, Neel and Zinsley, Dylan and Liu, Emily and Tennien, Will and Rudra, Atri and Zou, James and Mirhoseini, Azalia and others},
	year         = 2026,
	booktitle    = {ICLR}
}

@article{kujanpaa2024knowledge,
	title        = {Efficient Knowledge Injection in {LLM}s via Self-Distillation},
	author       = {Kalle Kujanp{\"a}{\"a} and Pekka Marttinen and Harri Valpola and Alexander Ilin},
	year         = 2025,
	journal      = {TMLR}
}

@inproceedings{cao2025infiniteicl,
	title        = {InfiniteICL: Breaking the Limit of Context Window Size via Long Short-term Memory Transformation},
	author       = {Cao, Bowen and Cai, Deng and Lam, Wai},
	year         = 2025,
	booktitle    = {ACL}
}

@article{scheurer2023training,
	title        = {Training language models with language feedback at scale},
	author       = {Scheurer, J{\'e}r{\'e}my and Campos, Jon Ander and Korbak, Tomasz and Chan, Jun Shern and Chen, Angelica and Cho, Kyunghyun and Perez, Ethan},
	year         = 2023,
	journal      = {arXiv preprint arXiv:2303.16755}
}

@inproceedings{dou2024re,
	title        = {Re-rest: Reflection-reinforced self-training for language agents},
	author       = {Dou, Zi-Yi and Yang, Cheng-Fu and Wu, Xueqing and Chang, Kai-Wei and Peng, Nanyun},
	year         = 2024,
	booktitle    = {EMNLP}
}

@inproceedings{zhou2025expo,
	title        = {ExPO: Unlocking Hard Reasoning with Self-Explanation-Guided Reinforcement Learning},
	author       = {Zhou, Ruiyang and Li, Shuozhe and Zhang, Amy and Leqi, Liu},
	year         = 2025,
	booktitle    = {NeurIPS}
}

@article{mitra2025semantic,
	title        = {Semantic Soft Bootstrapping: Long Context Reasoning in LLMs without Reinforcement Learning},
	author       = {Mitra, Purbesh and Ulukus, Sennur},
	year         = 2025,
	journal      = {arXiv preprint arXiv:2512.05105}
}

@article{qu2026pope,
	title        = {POPE: Learning to Reason on Hard Problems via Privileged On-Policy Exploration},
	author       = {Qu, Yuxiao and Setlur, Amrith and Smith, Virginia and Salakhutdinov, Ruslan and Kumar, Aviral},
	year         = 2026,
	journal      = {arXiv preprint arXiv:2601.18779}
}

@inproceedings{chen2025retrospective,
	title        = {Retrospective In-Context Learning for Temporal Credit Assignment with Large Language Models},
	author       = {Chen, Wentse and Chen, Jiayu and Tajwar, Fahim and Zhu, Hao and Duan, Xintong and Salakhutdinov, Ruslan and Schneider, Jeff},
	year         = 2025,
	booktitle    = {NeurIPS}
}

@article{penaloza2026privileged,
	title        = {Privileged Information Distillation for Language Models},
	author       = {Penaloza, Emiliano and Vattikonda, Dheeraj and Gontier, Nicolas and Lacoste, Alexandre and Charlin, Laurent and Caccia, Massimo},
	year         = 2026,
	journal      = {arXiv preprint arXiv:2602.04942}
}

@inproceedings{ouyang2022training,
	title        = {Training language models to follow instructions with human feedback},
	author       = {Ouyang, Long and Wu, Jeffrey and Jiang, Xu and Almeida, Diogo and Wainwright, Carroll and Mishkin, Pamela and Zhang, Chong and Agarwal, Sandhini and Slama, Katarina and Ray, Alex and others},
	year         = 2022,
	booktitle    = {NeurIPS}
}

@inproceedings{brown2020language,
	title        = {Language models are few-shot learners},
	author       = {Brown, Tom and Mann, Benjamin and Ryder, Nick and Subbiah, Melanie and Kaplan, Jared D and Dhariwal, Prafulla and Neelakantan, Arvind and Shyam, Pranav and Sastry, Girish and Askell, Amanda and others},
	year         = 2020,
	booktitle    = {NeurIPS}
}

@inproceedings{wei2022chain,
	title        = {Chain-of-thought prompting elicits reasoning in large language models},
	author       = {Wei, Jason and Wang, Xuezhi and Schuurmans, Dale and Bosma, Maarten and Xia, Fei and Chi, Ed and Le, Quoc V and Zhou, Denny and others},
	year         = 2022,
	booktitle    = {NeurIPS}
}

@article{chung2024scaling,
	title        = {Scaling instruction-finetuned language models},
	author       = {Chung, Hyung Won and Hou, Le and Longpre, Shayne and Zoph, Barret and Tay, Yi and Fedus, William and Li, Yunxuan and Wang, Xuezhi and Dehghani, Mostafa and Brahma, Siddhartha and others},
	year         = 2024,
	journal      = {JMLR},
	volume       = 25,
	number       = 70,
	pages        = {1--53}
}

@inproceedings{lin2022truthfulqa,
	title        = {Truthfulqa: Measuring how models mimic human falsehoods},
	author       = {Lin, Stephanie and Hilton, Jacob and Evans, Owain},
	year         = 2022,
	booktitle    = {ACL}
}

@inproceedings{zellers2019hellaswag,
	title        = {HellaSwag: Can a Machine Really Finish Your Sentence?},
	author       = {Zellers, Rowan and Holtzman, Ari and Bisk, Yonatan and Farhadi, Ali and Choi, Yejin},
	year         = 2019,
	booktitle    = {ACL}
}

@inproceedings{talmor2019commonsenseqa,
	title        = {Commonsenseqa: A question answering challenge targeting commonsense knowledge},
	author       = {Talmor, Alon and Herzig, Jonathan and Lourie, Nicholas and Berant, Jonathan},
	year         = 2019,
	booktitle    = {NAACL}
}

@inproceedings{wang2024helpsteer,
	title        = {Helpsteer 2: Open-source dataset for training top-performing reward models},
	author       = {Wang, Zhilin and Dong, Yi and Delalleau, Olivier and Zeng, Jiaqi and Shen, Gerald and Egert, Daniel and Zhang, Jimmy and Sreedhar, Makesh Narsimhan and Kuchaiev, Oleksii},
	year         = 2024,
	booktitle    = {NeurIPS}
}

@inproceedings{stienon2020learning,
	title        = {Learning to summarize from human feedback},
	author       = {Nisan Stiennon and Long Ouyang and Jeff Wu and Daniel M. Ziegler and Ryan Lowe and Chelsea Voss and Alec Radford and Dario Amodei and Paul Christiano},
	year         = 2020,
	booktitle    = {NeurIPS}
}

@misc{ica2026,
	title        = {Intrinsic Credit Assignment for Long Horizon Interaction},
	author       = {Auzina, Ilze Amanda and Strüber, Joschka and Hernández-Gutiérrez, Sergio and Goel, Shashwat and Prabhu, Ameya and Bethge, Matthias},
	year         = 2026,
	journal      = {arXiv preprint arXiv:2602.12342}
}

@article{don2025future,
	title        = {The future of open human feedback},
	author       = {Don-Yehiya, Shachar and Burtenshaw, Ben and Fernandez Astudillo, Ramon and Osborne, Cailean and Jaiswal, Mimansa and Kuo, Tzu-Sheng and Zhao, Wenting and Shenfeld, Idan and Peng, Andi and Yurochkin, Mikhail and others},
	year         = 2025,
	journal      = {Nature Machine Intelligence},
	volume       = 7,
	number       = 6,
	pages        = {825--835}
}

@inproceedings{stephan2024rlvf,
	title        = {Rlvf: Learning from verbal feedback without overgeneralization},
	author       = {Stephan, Moritz and Khazatsky, Alexander and Mitchell, Eric and Chen, Annie S and Hsu, Sheryl and Sharma, Archit and Finn, Chelsea},
	year         = 2024,
	booktitle    = {ICML}
}

@article{lee2024reinforcement,
	title        = {Reinforcement learning from reflective feedback (rlrf): Aligning and improving llms via fine-grained self-reflection},
	author       = {Lee, Kyungjae and Hwang, Dasol and Park, Sunghyun and Jang, Youngsoo and Lee, Moontae},
	year         = 2024,
	journal      = {arXiv preprint arXiv:2403.14238}
}

@inproceedings{madaan2023self,
	title        = {Self-refine: Iterative refinement with self-feedback},
	author       = {Madaan, Aman and Tandon, Niket and Gupta, Prakhar and Hallinan, Skyler and Gao, Luyu and Wiegreffe, Sarah and Alon, Uri and Dziri, Nouha and Prabhumoye, Shrimai and Yang, Yiming and others},
	year         = 2023,
	booktitle    = {NeurIPS}
}

@inproceedings{shinn2023reflexion,
	title        = {Reflexion: Language agents with verbal reinforcement learning},
	author       = {Shinn, Noah and Cassano, Federico and Gopinath, Ashwin and Narasimhan, Karthik and Yao, Shunyu},
	year         = 2023,
	booktitle    = {NeurIPS}
}

@inproceedings{yao2023retroformer,
	title        = {Retroformer: Retrospective large language agents with policy gradient optimization},
	author       = {Yao, Weiran and Heinecke, Shelby and Niebles, Juan Carlos and Liu, Zhiwei and Feng, Yihao and Xue, Le and Murthy, Rithesh and Chen, Zeyuan and Zhang, Jianguo and Arpit, Devansh and others},
	year         = 2024,
	booktitle    = {ICLR}
}

@article{yuksekgonul2025optimizing,
	title        = {Optimizing generative AI by backpropagating language model feedback},
	author       = {Yuksekgonul, Mert and Bianchi, Federico and Boen, Joseph and Liu, Sheng and Lu, Pan and Huang, Zhi and Guestrin, Carlos and Zou, James},
	year         = 2025,
	journal      = {Nature},
	volume       = 639,
	pages        = {609--616}
}

@inproceedings{goyal2019using,
	title        = {Using natural language for reward shaping in reinforcement learning},
	author       = {Goyal, Prasoon and Niekum, Scott and Mooney, Raymond J},
	year         = 2019,
	booktitle    = {IJCAI}
}

@inproceedings{xie2024text2reward,
	title        = {Text2reward: Reward shaping with language models for reinforcement learning},
	author       = {Xie, Tianbao and Zhao, Siheng and Wu, Chen Henry and Liu, Yitao and Luo, Qian and Zhong, Victor and Yang, Yanchao and Yu, Tao},
	year         = 2024,
	booktitle    = {ICLR}
}

@inproceedings{urcelay2025words,
	title        = {From Words to Rewards: Leveraging Natural Language for Reinforcement Learning},
	author       = {Urcelay, Belen Martin and Krause, Andreas and Ramponi, Giorgia},
	year         = 2026,
	booktitle    = {TMLR}
}

@inproceedings{wang2025text2grad,
	title        = {Text2Grad: Reinforcement Learning from Natural Language Feedback},
	author       = {Wang, Hanyang and Wang, Lu and Zhang, Chaoyun and Mao, Tianjun and Qin, Si and Lin, Qingwei and Rajmohan, Saravan and Zhang, Dongmei},
	year         = 2026,
	booktitle    = {ICLR}
}

@article{gunjal2025rubrics,
	title        = {Rubrics as rewards: Reinforcement learning beyond verifiable domains},
	author       = {Gunjal, Anisha and Wang, Anthony and Lau, Elaine and Nath, Vaskar and He, Yunzhong and Liu, Bing and Hendryx, Sean},
	year         = 2025,
	journal      = {arXiv preprint arXiv:2507.17746}
}

@article{team2025kimi,
	title        = {Kimi k2: Open agentic intelligence},
	author       = {{Kimi Team} and Bai, Yifan and Bao, Yiping and Chen, Guanduo and Chen, Jiahao and Chen, Ningxin and Chen, Ruijue and Chen, Yanru and Chen, Yuankun and Chen, Yutian and others},
	year         = 2025,
	journal      = {arXiv preprint arXiv:2507.20534}
}

@inproceedings{rafailov2023direct,
	title        = {Direct preference optimization: Your language model is secretly a reward model},
	author       = {Rafailov, Rafael and Sharma, Archit and Mitchell, Eric and Manning, Christopher D and Ermon, Stefano and Finn, Chelsea},
	year         = 2023,
	booktitle    = {NeurIPS}
}

@article{don2024naturally,
	title        = {Naturally occurring feedback is common, extractable and useful},
	author       = {Don-Yehiya, Shachar and Choshen, Leshem and Abend, Omri},
	year         = 2024,
	journal      = {arXiv preprint arXiv:2407.10944}
}

@article{shi2026experiential,
	title        = {Experiential Reinforcement Learning},
	author       = {Shi, Taiwei and Chen, Sihao and Jiang, Bowen and Song, Linxin and Yang, Longqi and Zhao, Jieyu},
	year         = 2026,
	journal      = {arXiv preprint arXiv:2602.13949}
}

@article{shao2025dr,
	title        = {Dr tulu: Reinforcement learning with evolving rubrics for deep research},
	author       = {Shao, Rulin and Asai, Akari and Shen, Shannon Zejiang and Ivison, Hamish and Kishore, Varsha and Zhuo, Jingming and Zhao, Xinran and Park, Molly and Finlayson, Samuel G and Sontag, David and others},
	year         = 2025,
	journal      = {arXiv preprint arXiv:2511.19399}
}
